\documentclass{article}
\usepackage[dblblindworkshop,final]{neurips_2026}
\workshoptitle{Transitioning from Pre-Training to Post-Training}

\usepackage[utf8]{inputenc}
\usepackage[T1]{fontenc}
\usepackage{algorithm}
\usepackage{algpseudocode}
\usepackage{hyperref}
\usepackage{url}
\usepackage{booktabs}
\usepackage{amsfonts}
\usepackage{amsmath}
\usepackage{mathtools}
\usepackage{amssymb}
\usepackage{amsthm}
\usepackage{nicefrac}
\usepackage{microtype}
\usepackage{xcolor}
\usepackage{colortbl}
\usepackage{tabularx}
\usepackage{graphicx}
\usepackage{multirow}
\usepackage{makecell}
\usepackage{placeins}
\usepackage{tikz}
\usetikzlibrary{arrows.meta,calc,shapes.geometric}

\makeatletter
\newcommand{\sidecaption}[1]{\def\@captype{#1}\caption}
\makeatother

\newcommand{\method}{ProbePO}
\newcommand{\Vhat}{\hat{V}}
\newcommand{\E}{\mathbb{E}}
\newcommand{\Var}{\mathrm{Var}}

\newtheorem{proposition}{Proposition}

\title{\method{}: Sparse Probes, Redistributed Credit for Critic-Free LLM Reinforcement Learning}

\author{%
  Diyi Hu\\
  Independent Researcher\\
  \texttt{diyihu.hu@gmail.com}
  \And
  Ruizhong Qiu\\
  University of Illinois Urbana-Champaign
  \And
  Lin Qiu\\
  Meta
  \And
  Hanghang Tong\\
  University of Illinois Urbana-Champaign
  \And
  Hanqing Zeng\\
  Meta
}

\begin{document}

\maketitle

\begin{abstract}
Reinforcement learning from verifiable rewards assigns one terminal reward to an entire reasoning trajectory, obscuring how individual reasoning steps contributed to the outcome.
GRPO broadcasts a group-normalized advantage to every response token, while finer-grained approaches often rely on dense auxiliary rollouts or branching schemes that can also alter the data used for policy updates.
We cast the missing middle as \emph{budgeted measurement}: with a limited rollout budget, which partial solutions should be evaluated, and how should their value estimates refine credit?
\method{} selects partial-solution states with a reward-free uncertainty score and estimates their expected verifier rewards under a capped continuation budget.
\method{} redistributes, rather than replaces, GRPO's trajectory-level advantage across tokens while preserving each trajectory's mean token-level advantage; probe continuations serve only as measurements.
Across three DeepSeekMath and Qwen-family models spanning 1.5B--7B parameters, \method{} consistently outperforms GRPO and achieves the highest accuracy among the four evaluated methods in 16 of 18 model--benchmark comparisons.
Against dense-probing VinePPO on DeepSeekMath-7B, \method{} achieves higher accuracy on all six benchmarks with 76\% less end-to-end training time.
\end{abstract}

\section{Introduction}
\label{sec:intro}

Reinforcement learning from verifiable rewards (RLVR) has become a central approach for improving the mathematical and coding reasoning of large language models \citep{shao2024deepseekmath, guo2025deepseekr1}.
Its supervision is nevertheless coarse: a verifier evaluates the completed response, leaving the contribution of intermediate reasoning implicit.
Group Relative Policy Optimization (GRPO) \citep{shao2024deepseekmath} assigns each completion a group-normalized advantage and broadcasts it to every response token.
This assignment is simple and effective, but it ignores the different roles that tokens play within a trajectory.
Under a positive trajectory-level advantage, an early detour in the reasoning trajectory and the later step that repairs it are reinforced equally; under a negative advantage, sound progress is penalized together with the mistake that invalidates it.

A response to a difficult prompt typically unfolds over multiple reasoning steps, creating natural intermediate points at which progress can be assessed.
Intermediate value measurements at these reasoning boundaries offer a direct way to refine this coarse allocation.
At such a boundary, we hold the generated prefix fixed, sample several possible suffixes from the current policy, and score the resulting completed responses with the verifier.
We call each sampled suffix (from the boundary through a final answer) a \emph{probe continuation}; averaging their rewards estimates the partial solution's expected terminal verifier reward under the current policy.
Differences between value estimates at successive boundaries indicate whether expected terminal verifier reward rose or fell across the intervening segment.
Evidence suggests that this expectation changes little across many reasoning steps, while a smaller number change it sharply \citep{kazemnejad2025vineppo}.
Because every probe continuation requires additional generation, measuring every eligible boundary can be costly.
A fixed continuation budget therefore creates a measurement-allocation problem:
\begin{center}
\emph{Which partial solutions are worth measuring, and how should those\\
measurements redistribute credit along the original trajectory?}
\end{center}

We present \method{} (\emph{probe policy optimization}), which adds a budgeted measurement layer to GRPO (Figure~\ref{fig:overview}).
It first ranks candidate reasoning-step boundaries with a lightweight, reward-free uncertainty score, then spends a capped probe budget on a few top-ranked candidates.
At each selected boundary, \method{} samples continuations directly from the rollout policy and averages their verifier rewards to estimate the pivot prefix's expected verifier reward.
These probe continuations serve exclusively as measurements, allowing \method{} to refine credit.
To improve the precision of each sparse estimate, \method{} regularizes it toward GRPO's group mean $\mu_g$.
It then uses differences between successive prefix-value estimates as segment-wise adjustments.
Rather than replacing GRPO's trajectory-level advantage, \method{} adds these adjustments as a zero-average redistribution across the response.
This preserves each trajectory's mean token-level advantage at its GRPO value while allowing different reasoning segments to receive different emphasis.

\paragraph{Contributions.}
\begin{enumerate}
  \item \textbf{Budgeted credit measurement.} We formulate intermediate credit assignment as allocating a finite continuation budget across partial-solution states, separating where to measure from how sparse measurements are converted into credit.
  \item \textbf{GRPO-anchored redistribution.} We introduce \method{}, combining reward-free pivot selection, regularized continuation-based value estimates, and a centered segment-wise correction. The correction exactly preserves each trajectory's mean token-level advantage at its GRPO value and recovers GRPO when no pivot is selected.
  \item \textbf{Accuracy and efficiency.} Across three models and six benchmarks, \method{} consistently outperforms GRPO. Controlled studies support both structured placement and shrinkage. At matched training time, \method{} remains more accurate than GRPO; against dense-probing VinePPO, it achieves higher accuracy on all six benchmarks with 76\% less wall-clock training time.
\end{enumerate}

\begin{figure}[t]
  \centering
  \noindent\makebox[\linewidth][c]{%
    \hspace*{-0.05\linewidth}%
    \resizebox{1.12\linewidth}{!}{\input{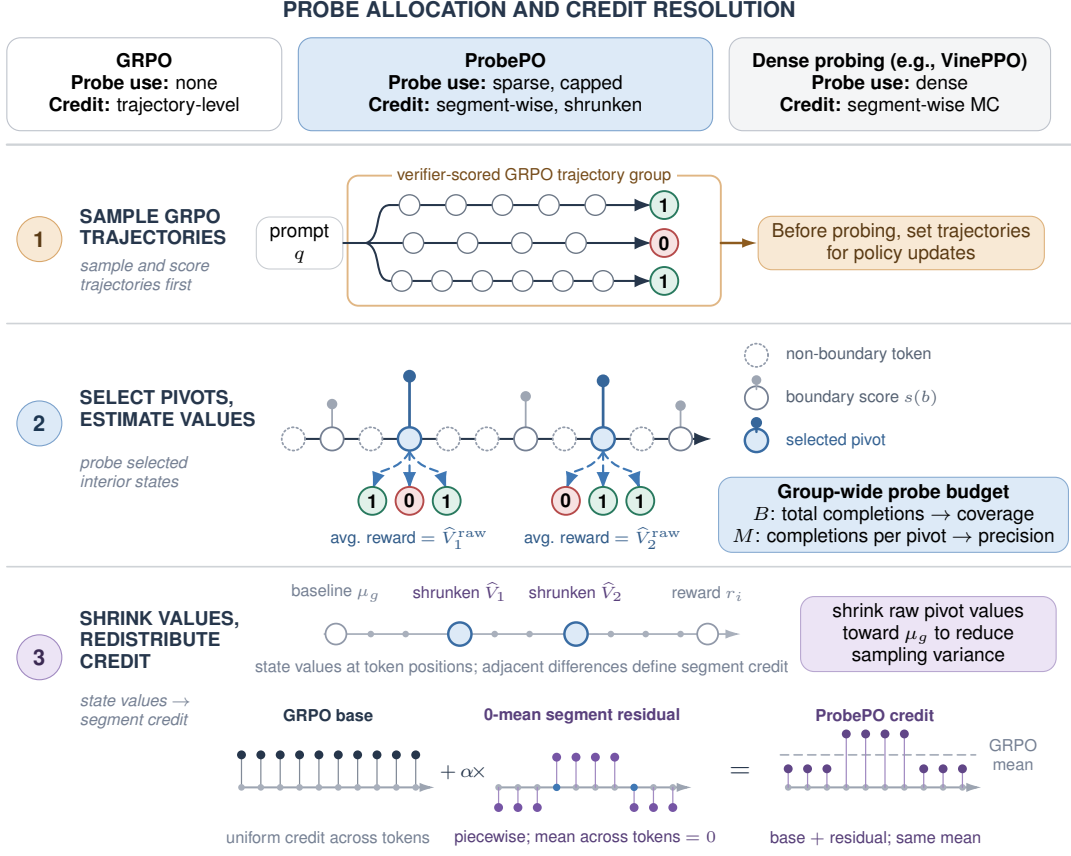}}%
    \hspace*{0.05\linewidth}}
  \caption{\textbf{\method{} improves cost-quality tradeoff by assigning within-trajectory credit with a sparse-probe budget.}
  GRPO incurs no auxiliary-probe overhead but provides only trajectory-level credit. Dense per-boundary probing, exemplified by VinePPO, provides fine resolution at high rollout cost, while finite-sample Monte Carlo differences remain noise-sensitive. \method{} instead caps the auxiliary-probe budget, concentrates it on selected pivots, and shrinks their value estimates before deriving segment-wise credit. The pipeline first fixes the policy-update trajectories, then ranks and probes candidate boundaries under the group-wide budget, and finally adds a token-mean-zero segment residual to the GRPO base advantage.}
  \label{fig:overview}
\end{figure}

\section{Preliminaries}
\label{sec:prelim}

\paragraph{Setup and GRPO.}
A policy $\pi_\theta$ generates a continuation $y=(y_1,\dots,y_T)$ for a prompt $q$; a verifier returns an undiscounted terminal reward $R\in[0,1]$.
Let $s_t\coloneqq(q,y_{1:t})$ denote the partial-solution state after $t$ generated tokens, with $s_0=q$.
Training optimizes a clipped PPO surrogate \citep{schulman2017ppo} over per-token advantages $A_{i,t}$.
We use \emph{advantage} for the action-value contrast that weights a policy-gradient update, \emph{credit} for its allocation across tokens, and \emph{residual} for the zero-mean correction introduced in Section~\ref{sec:estimator}; a process reward is an intermediate reward signal rather than a synonym for advantage.
For each prompt, GRPO \citep{shao2024deepseekmath} samples a group of $G$ continuations with rewards $r_1,\dots,r_G$ and assigns continuation $i$ the group-normalized advantage
\begin{equation}
A_i = \frac{r_i-\mu_g}{\sigma_g},\qquad \mu_g=\tfrac1G\sum_j r_j,\quad \sigma_g=\mathrm{std}(\{r_j\})+\varepsilon,
\label{eq:grpo}
\end{equation}
broadcast to every token: $A_{i,t}=A_i$ for all $t$.
The group mean $\mu_g$ provides GRPO's prompt-level endpoint anchor; the standardized return contrast does not identify where within the continuation that contrast arose.

\paragraph{Token-level credit as value increments.}
For a fixed policy, let $(\mathcal F_t)_{t=0}^{T}$ be the filtration generated by the prompt and successive tokens, so $\mathcal F_t$ contains the information in $s_t=(q,y_{1:t})$.
Define $V_t^\pi\coloneqq\mathbb E_\pi[R\mid\mathcal F_t]$.
By the tower property, $\{V^\pi_t\}_{t=0}^T$ is a Doob martingale, with $V^\pi_0=\E_\pi[R\mid q]$ and $V^\pi_T=R$ when the completed response determines the verifier score.
Because appending token $y_t$ deterministically maps $s_{t-1}$ to $s_t$ and reward is terminal, $Q^\pi(s_{t-1},y_t)=V^\pi_t$.
The token-level action advantage, denoted $A^\pi(s_{t-1},y_t)$, is therefore
\begin{equation}
A^\pi(s_{t-1},y_t)=V^\pi_t-V^\pi_{t-1},
\qquad
\sum_{t=1}^{T}A^\pi(s_{t-1},y_t)=R-V^\pi_0.
\label{eq:telescope}
\end{equation}
GRPO does not estimate these tokenwise increments.
Instead, it assigns the trajectory-level advantage $A_i=(r_i-\mu_g)/\sigma_g$ to every response token: $A_{i,t}=A_i$.
VinePPO takes the dense alternative, estimating values at every eligible reasoning-step boundary and differencing adjacent estimates to construct step-level credit \citep{kazemnejad2025vineppo}.

\paragraph{What should a placement signal track?}
For a state $s$ and $a\sim\pi(\cdot\mid s)$, define $Q^\pi(s,a)\coloneqq\E_\pi[R\mid s,a]$.
The law of total variance gives
\begin{equation}
\Var_\pi(R\mid s)
=\underbrace{\E_{a\sim\pi(\cdot\mid s)}[\Var_\pi(R\mid s,a)]}_{\text{continuation variance}}
+\underbrace{\Var_{a\sim\pi(\cdot\mid s)}(Q^\pi(s,a))}_{\text{action-fork variance}}.
\label{eq:trichotomy}
\end{equation}
The action-fork term equals $\E_{a\sim\pi(\cdot\mid s)}[A^\pi(s,a)^2]$ and is large where the next-token choice can materially change expected terminal reward, making it a natural reference target for evaluating probe-placement signals.
By contrast, $\Var_\pi(R\mid s)$ determines the sampling variance of an $M$-continuation prefix-value estimate, $\Var(\Vhat_{\mathrm{raw}}(s)\mid s)=\Var_\pi(R\mid s)/M$.

\section{\method{}}
\label{sec:method}

\method{} addresses the budgeted-measurement problem posed in Section~\ref{sec:intro} through two layers. The acquisition layer selects boundaries and allocates the probe budget; the estimation layer converts the resulting sparse prefix-value estimates into token-level advantages.
A selected boundary is a \emph{pivot}, and a verifier-scored suffix generated from it is a \emph{probe continuation}.
Probe continuations serve only as measurements; the actor batch remains fixed.
Figure~\ref{fig:overview} illustrates the pipeline.

\subsection{The reward-free selector}
\label{sec:cascade}

\paragraph{Candidate boundaries.}
Within each trajectory we detect candidate boundaries deterministically at reasoning-step delimiters (e.g., paragraph breaks, enumerated step markers), discard any terminal boundary $t_b=T$, and merge boundaries closer than a minimum token gap.

\paragraph{Why reward-free?}
Action-fork variance (Equation~\ref{eq:trichotomy}) serves as a reference target rather than an online selection signal: estimating it at every candidate would either require extra verifier-scored continuations solely for selection or, if those continuations were reused for value estimation, risk post-selection bias.
We therefore rank candidate boundaries using the reward-free forward-window maximum entropy score $H_{\max}$, computed from backbone logits without auxiliary rollouts or reward queries.
Section~\ref{sec:frontier} evaluates how well this score tracks action-fork variance empirically.

\paragraph{Score.}
For a backbone state after token $t$, define its full-vocabulary next-token entropy $H_t$; each candidate boundary $b$ at position $t_b$ receives
\begin{equation}
\begin{aligned}
H_t
  &= -\sum_{v\in\mathcal V}\pi_\theta(v\mid q,y_{1:t})
       \log \pi_\theta(v\mid q,y_{1:t}),\\
H_{\max}(b)
  &= \max_{\substack{0\leq \ell<L\\ t_b+\ell<T}} H_{t_b+\ell},\\
s(b)
  &= H_{\max}(b)\,w_{\mathrm{pos}}(t_b/T).
\end{aligned}
\label{eq:selector}
\end{equation}
Here $L=8$ in every reported arm, and the maximum is evaluated on the realized backbone tokens.
The optional factor $w_{\mathrm{pos}}(u)$ can encode a preference over normalized boundary position $u=t_b/T$; all reported experiments use $w_{\mathrm{pos}}(u)\equiv1$, so this factor does not affect the ranking.
The forward maximum reduces sensitivity to exact boundary alignment by allowing a nearby high-entropy state to represent the boundary. Because it uses logits produced during backbone generation, scoring requires no additional rollouts.
The top fraction of candidates per trajectory forms the shortlist for group-pooled allocation.

\paragraph{Group-pooled budget.}
We define budget $B$ as the number of probe continuations across all $G$ trajectories.
We rank shortlisted boundaries jointly and select $\lfloor B/M\rfloor$ pivots.
Each pivot is evaluated with $M$ independent continuations.
Pooling is deliberate: a trajectory rich in high-ranked boundaries may receive several pivots while a flat one receives none.

\paragraph{Probing.}
At each selected pivot $\tau$ in trajectory $i$ we freeze the prefix $y_{i,1:\tau}$, sample $M$ probe continuations to completion from the same current policy $\pi_\theta$ that generated the backbone batch, score each with the same verifier that produces training rewards, and set $\Vhat_{\mathrm{raw}}(\tau)=\tfrac1M\sum_m \tilde r_m$.
The probe mean estimates the current policy's value $V^{\pi_\theta}(s_\tau)$ under the training verifier, provided sampling matches $\pi_\theta$, continuations terminate, and the verifier score is recoverable from the generated suffix (Appendix~\ref{app:forced_answer}).
Its standard error is at most $0.5/\sqrt{M}$ for conditionally independent rewards in $[0,1]$.
This finite-sample uncertainty motivates the shrinkage introduced in Section~\ref{sec:estimator}.

\subsection{GRPO-anchored credit redistribution}
\label{sec:estimator}

Consider trajectory $i$ from group $g$, with length $T_i$, reward $r_i$, group statistics $\mu_g,\sigma_g$, and $K\geq 0$ selected pivots at positions $0<\tau_1<\cdots<\tau_K<T_i$.
Let $A_{\mathrm{base},i}$ denote the trajectory-level base advantage assigned by the configured GRPO rule; under Eq.~\eqref{eq:grpo}, $A_{\mathrm{base},i}=(r_i-\mu_g)/\sigma_g$.
We add endpoint anchors $\tau_0=0$ and $\tau_{K+1}=T_i$, with values $\Vhat(\tau_0)=\mu_g$ and $\Vhat(\tau_{K+1})=r_i$.

\paragraph{Shrinkage.}
Each interior value averages $M$ probe returns.
We regularize this estimate toward the group baseline using linear pseudo-count shrinkage:
\begin{equation}
\Vhat(\tau_k)
=
\frac{M\,\Vhat_{\mathrm{raw}}(\tau_k)+n_0\,\mu_g}{M+n_0},
\qquad k=1,\ldots,K.
\label{eq:shrink}
\end{equation}
Here $n_0\geq 0$ controls the shrinkage strength, and $M/(M+n_0)$ is the weight on the raw estimate.
For binary verifier rewards, Eq.~\eqref{eq:shrink} is the posterior mean under a standard Beta--Bernoulli conjugate model \citep[Ch.~2]{gelman2013bayesian}; Appendix~\ref{app:proofs} gives this interpretation and the estimator's conditional bias--variance decomposition.

\paragraph{Segment-wise redistribution.}
The pivots partition the response into segments
$I_k=\{t:\tau_{k-1}<t\leq\tau_k\}$ with lengths
$\ell_k=\tau_k-\tau_{k-1}$.
For each segment, define the value increment
\[
d_k=\Vhat(\tau_k)-\Vhat(\tau_{k-1}),
\qquad k=1,\ldots,K+1.
\]
The endpoint anchors imply
$\sum_{k=1}^{K+1}d_k=r_i-\mu_g$.
When $K=1$, the two increments measure the change from the group baseline to the pivot and from the pivot to the observed return, respectively.

Let $k(t)$ denote the segment containing token $t$.
We broadcast each segment increment to its tokens as the raw token-level advantage
$A^{\mathrm{raw}}_{i,t}=d_{k(t)}/\sigma_g$, then subtract its
per-trajectory token mean:
\begin{equation}
\tilde A_{i,t}
=
A^{\mathrm{raw}}_{i,t}
-
\frac{1}{T_i}\sum_{s=1}^{T_i}A^{\mathrm{raw}}_{i,s},
\qquad
\sum_{t=1}^{T_i}\tilde A_{i,t}=0.
\label{eq:center}
\end{equation}
The final token-level advantage is
\begin{equation}
\boxed{
A_{i,t}
=
A_{\mathrm{base},i}
+
\alpha\,\tilde A_{i,t}
},
\label{eq:final}
\end{equation}
where $\alpha$ controls the magnitude of the segment-wise redistribution.
The GRPO base advantage determines the token mean of $A_{i,t}$, while the centered residual redistributes credit among the trajectory's segments.

\begin{proposition}[Token-mean preservation and recovery at zero pivots]
\label{prop:mean}
With $A_{\mathrm{base},i}$ set to the GRPO advantage in Eq.~\eqref{eq:grpo}, for every trajectory, pivot set, collection of interior estimates, and value of $\alpha$,
\[
\frac{1}{T_i}\sum_{t=1}^{T_i}A_{i,t}
=
\frac{r_i-\mu_g}{\sigma_g}.
\]
If $K=0$, then
$A_{i,t}=(r_i-\mu_g)/\sigma_g$ for every response token.
\end{proposition}
The centered residual preserves the trajectory's GRPO token-mean advantage while redistributing its token-level policy-gradient weights. Appendix~\ref{app:proofs} gives the proof and the matched conditions under which $K=0$ recovers the GRPO update.

Algorithm~\ref{alg:probepo} summarizes one \method{} training step combining both layers.
\begin{algorithm}[H]
\caption{\method{} training step}
\label{alg:probepo}
\begin{algorithmic}[1]
\State \textbf{Input:} prompts and current policy $\pi_\theta$
\State \textbf{Parameters:} group size $G$, probe budget $B$, continuations per pivot $M$, shrinkage $n_0$, and residual weight $\alpha$
\State \textbf{For each prompt:} sample $G$ backbone completions from $\pi_\theta$; verify $\to$ rewards $r_i$
\For{each group $g$}
  \State Compute $(\mu_g,\sigma_g)$ and the configured base advantages $A_{\mathrm{base},i}$
  \For{each trajectory $i$ in group $g$}
    \State Detect candidate boundaries, score them by Eq.~\eqref{eq:selector}, and retain the top fraction
  \EndFor
    \State Rank shortlisted candidates and select $\lfloor B/M\rfloor$ pivots. \label{algl:fund}
  \For{each selected pivot $(i,\tau)$}
    \State Generate $M$ fresh probe continuations from $\pi_\theta$; verify and average rewards $\to \Vhat_{\mathrm{raw}}(\tau)$ \label{algl:probe}
  \EndFor
  \For{each trajectory $i$ in group $g$}
    \State Set the endpoint anchors and shrink interior values by Eq.~\eqref{eq:shrink} \label{algl:estb}
    \State Compute $\tilde A_{i,t}$ and $A_{i,t}$ by Eqs.~\eqref{eq:center}--\eqref{eq:final} \label{algl:este}
  \EndFor
\EndFor
\State Update $\theta$ with the standard GRPO objective, using $A_{i,t}$ in place of GRPO's broadcast trajectory-level advantage
\end{algorithmic}
\end{algorithm}

\section{Related work}
\label{sec:related}

\paragraph{Trajectory-level credit assignment.} GRPO \citep{shao2024deepseekmath} and GSPO \citep{chujie2025group} broadcast a group-relative advantage across each completion. DAPO \citep{yu2025dapo} removes zero-variance groups after rollout, whereas GRESO \citep{zheng2025greso} skips predicted zero-variance prompts beforehand; Dr.~GRPO \citep{liu2025drgrpo} removes length normalization and group standardization. Recent variants rescale or transform trajectory advantages ($\lambda$-GRPO \citep{wang2025lambda}; $f$-GRPO \citep{haldar2026f}), iteratively refine trajectories (iGRPO \citep{hatamizadeh2026igrpo}), or prioritize replay using trajectory rewards \citep{yoo2026rollout}. A trajectory-level signal cannot distinguish helpful from harmful regions of the same response; \method{} resolves this while retaining GRPO's trajectory-level anchor.

\paragraph{Token- and step-level credit assignment.} Process reward models (PRMs) learn dense step feedback from labeled or automatically constructed process data \citep{wang2024mathshepherd,setlur2025rewarding}. RED differences a sequence reward model's scores on consecutive prefixes to obtain token rewards, whereas T-REG regularizes implicit token rewards toward targets produced by contrastive prompts \citep{li2025red,zhou2025t}. VinePPO instead estimates policy-matched values by probing every reasoning-step boundary \citep{kazemnejad2025vineppo}. Together, these methods expose a supervision--cost trade-off: learned or self-generated supervision adds label construction, model inference, or prompted-scoring costs and may not match current-policy verifier success, while dense policy-matched probing incurs high rollout cost. \method{} targets the middle ground with verifier-grounded credit under a capped sparse-measurement budget and no separate reward model.

\paragraph{Segment-level credit assignment.} Segment methods assign shared credit within contiguous spans. SPO partitions trajectories at fixed intervals or low-probability cutpoints, estimates boundary values from policy continuations, and differences adjacent MC means \citep{guo2025spo}. Other approaches derive segment scores from Shapley or intervention-based attribution, generative PRMs, or gold-answer probabilities \citep{cao2025scar,xie2025capo,yang2026int,wang2026grpo,li2026ssvpo}. SPO provides current-policy segment values, but its rollout cost grows with the number of probed boundaries. \method{} instead concentrates a fixed group-wide budget on high-ranked boundaries, regularizes the resulting estimates, and centers the segment correction so local refinements preserve the trajectory's overall GRPO credit.

\paragraph{Tree-based credit assignment.} Online tree methods obtain intermediate evidence by branching trajectories and commonly reuse descendants as policy-update samples \citep{guo2025spo,yang2025treerpo,li2025treepo,zhao2026reinforced}. TEMPO constructs a prefix tree post hoc from an ordinary GRPO batch \citep{tran2025exploiting}, while OmegaPRM allocates trees offline to construct PRM labels \citep{luo2024omegaprm}. Online reuse amortizes branching cost, but branch placement and fanout affect both value evidence and which samples enter the actor update, making credit gains difficult to separate from exploration and data reuse. \method{} also branches, but uses a capped sparse continuation budget only for measurement while leaving the backbone actor batch unchanged (Appendix~\ref{app:trees}).

\section{Experiments}
\label{sec:experiments}

We ask four questions: Does \method{} improve reasoning accuracy across model
families? Do structured placement and shrinkage contribute to its performance?
How do its accuracy and training cost compare with GRPO and dense probing? Does
the reward-free selector identify informative boundaries?

\subsection{Experimental setup}
\label{sec:setup}
\label{sec:protocol}

\paragraph{Backbones and baselines.}
We evaluate three complementary backbones spanning 1.5B--7B parameters: the math-specialized DeepSeekMath-7B-SFT \citep{shao2024deepseekmath}, the compact reasoning-distilled R1-Distill-Qwen-1.5B \citep{guo2025deepseekr1}, and Qwen3-4B-Base \citep{qwen3_2025}. GRPO provides the trajectory-level reference.
SPO-tree estimates segment credit from branching rollouts
\citep{guo2025spo}, while TreePO uses tree-structured exploration and credit
assignment \citep{li2025treepo}. For each model, GRPO, SPO-tree, TreePO, and
\method{} use the same training problems, evaluation protocol, context limits,
and training horizon, enabling direct within-model comparability. We evaluate VinePPO
separately as the dense policy-matched baseline, which estimates values at
every eligible reasoning-step boundary \citep{kazemnejad2025vineppo}.
Tables~\ref{tab:probepo-shared-recipe}--\ref{tab:probepo-decoding-recipe}
in Appendix~\ref{app:repro} give the headline \method{} recipes and evaluation
protocol.

\paragraph{Training data.}
Following established RL-on-math protocols for each model family, DeepSeekMath
and R1 use VinePPO's PRM800K/MATH repartition
\citep{lightman2024lets,kazemnejad2025vineppo}, while Qwen3 uses the original
MATH training split adopted by ROSE \citep{zhao2026reinforced}.  All four
algorithms for a given model train on the same problem identities.
Appendix~\ref{app:repro} specifies the split construction, filtering, and
prompt templates.

\paragraph{Metrics and checkpoint rule.}
Mean@$k$ averages correctness over $k$ sampled responses per problem and then across problems. We use $k=16$ for MATH-500, OlympiadBench, CollegeMath, AMC, and MinervaMath, and $k=32$ for AIME24. \textbf{Avg-6} is the unweighted mean of all six reported dataset-level point estimates. Table~\ref{tab:main} reports each method's terminal checkpoint at the shared within-model training horizon: step 1{,}000 for DeepSeekMath and R1, and step 27 for Qwen3.

\subsection{Main results}
\label{sec:mainresults}

\begin{table}[t]
\centering
\caption{\textbf{Terminal accuracy (\%) across three backbones and six mathematical reasoning benchmarks.}
\textbf{Avg-6} is the unweighted benchmark mean; bold marks the best score within each backbone.}
\label{tab:main}
\small
\definecolor{tablegroup}{RGB}{244,245,247}
\definecolor{tableours}{RGB}{238,246,241}
\definecolor{tableavg}{RGB}{247,247,247}
\renewcommand{\arraystretch}{1.03}
\setlength{\tabcolsep}{2.4pt}
\begin{tabularx}{\linewidth}{@{}l*{6}{>{\centering\arraybackslash}X}>{\columncolor{tableavg}\centering\arraybackslash}X@{}}
\toprule
\textbf{Method} & \textbf{MATH-500} & \makecell[c]{\textbf{Olympiad}\\\textbf{Bench}} & \makecell[c]{\textbf{College}\\\textbf{Math}} & \textbf{AMC} & \makecell[c]{\textbf{Minerva}\\\textbf{Math}} & \makecell[c]{\textbf{AIME24}\\\textbf{(@32)}} & \textbf{Avg-6} \\
\midrule
\rowcolor{tablegroup}
\multicolumn{8}{l}{\textbf{DeepSeekMath-7B-SFT}} \\
GRPO & 42.55 & 12.67 & 22.57 & 14.38 & 11.79 & 1.77 & 17.62 \\
SPO-tree & 37.58 & 9.93 & 20.11 & 12.65 & 10.96 & 0.83 & 15.34 \\
TreePO & 43.86 & 11.73 & 22.46 & 11.22 & 9.38 & 0.83 & 16.58 \\
\rowcolor{tableours}
\textbf{\method{} (ours)} & \textbf{45.81} & \textbf{14.47} & \textbf{22.77} & \textbf{17.09} & \textbf{12.96} & \textbf{2.92} & \textbf{19.34} \\
\midrule
\rowcolor{tablegroup}
\multicolumn{8}{l}{\textbf{R1-Distill-Qwen-1.5B}} \\
GRPO & 80.10 & 46.08 & 55.45 & 54.89 & 29.69 & 12.71 & 46.49 \\
SPO-tree & 78.93 & 43.33 & 54.88 & 52.41 & 31.09 & \textbf{15.42} & 46.01 \\
TreePO & 80.36 & 46.94 & 54.73 & 53.92 & \textbf{31.18} & 14.69 & 46.97 \\
\rowcolor{tableours}
\textbf{\method{} (ours)} & \textbf{82.34} & \textbf{48.44} & \textbf{57.31} & \textbf{58.13} & 30.28 & 13.75 & \textbf{48.37} \\
\midrule
\rowcolor{tablegroup}
\multicolumn{8}{l}{\textbf{Qwen3-4B-Base}} \\
GRPO & 76.00 & 38.50 & 31.59 & 47.14 & 24.26 & 12.08 & 38.26 \\
SPO-tree & 75.93 & 37.86 & 31.44 & 47.82 & 23.85 & 13.70 & 38.43 \\
TreePO & 77.48 & 37.89 & 30.29 & 49.02 & 24.06 & 15.52 & 39.04 \\
\rowcolor{tableours}
\textbf{\method{} (ours)} & \textbf{79.40} & \textbf{41.13} & \textbf{31.99} & \textbf{50.53} & \textbf{26.10} & \textbf{17.50} & \textbf{41.11} \\
\bottomrule
\end{tabularx}
\end{table}

\paragraph{Overall accuracy.}
Table~\ref{tab:main} shows that, across all three models, \method{} achieves the highest Avg-6, exceeding GRPO by $+1.72$, $+1.88$, and $+2.85$ points on DeepSeekMath, R1-Distill-Qwen-1.5B, and Qwen3, respectively.
It outperforms GRPO in all 18 model--benchmark comparisons and ranks first among the four methods in 16 of 18, demonstrating broad gains across models and benchmarks.

\paragraph{Within-response credit.}
Figure~\ref{fig:credit-examples} shows how Eq.~\eqref{eq:final} redistributes
credit within individual responses. Column~(a) asks for three-step paths
between opposite cube vertices. The response detours into graph search in
$I_2$ before recovering the $3!=6$ permutation argument in $I_3$; \method{}
penalizes the detour and assigns the recovery the largest positive credit.
Column~(b) asks for $QA$ in a parallelogram ($AC=15$, ground truth $10$).
After deriving $10$, the response reuses the wrong ratio in $I_3$ and answers
$5$; the reversal receives the strongest penalty, while the earlier derivation
is penalized less than by GRPO. Together,
the examples show that \method{} distinguishes consequential segments that
GRPO treats uniformly while preserving the response's GRPO token mean.

\begin{figure}[t]
  \centering
  \includegraphics[width=\linewidth]{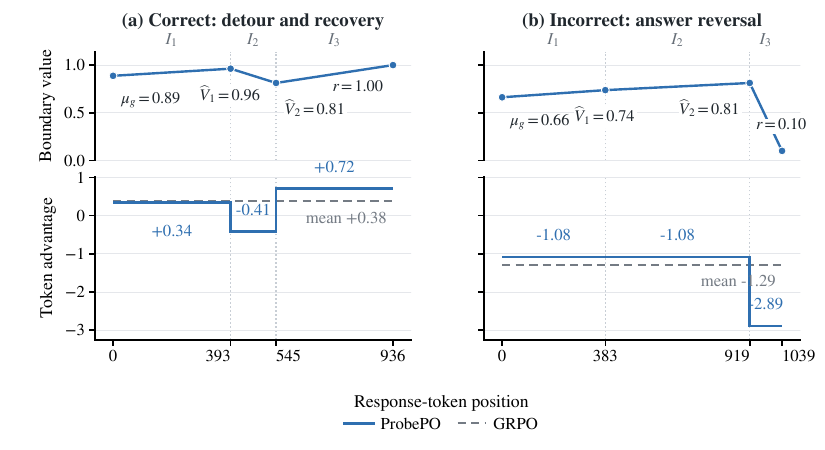}
  \caption{\textbf{Boundary values and segment credit.}
  Two held-out MATH responses from the trained R1-Distill-Qwen-1.5B checkpoint: (a) a correct cube-path response and (b) an incorrect parallelogram response. Top panels show the anchor, shrunk pivot estimates ($M=8,n_0=4$), and terminal reward; bottom panels compare \method{} and GRPO advantages on a shared scale, with widths proportional to token counts. The detour ($I_2$, a) receives negative credit; the answer reversal ($I_3$, b) receives the strongest penalty. Each \method{} profile preserves the GRPO token mean; in (b), $r=0.10$ denotes a boxed but incorrect answer. Appendix~\ref{app:credit-examples} provides the full contexts and pooled-$M=32$ check.}
  \label{fig:credit-examples}
\end{figure}

\FloatBarrier
\subsection{Controlled ablations: placement and shrinkage}
\label{sec:component-ablations}

\paragraph{Controlled design.}
We ablate the structured placement module and shrinkage in a controlled Qwen3-4B study.
All arms share the training and evaluation setup and retain every generated group; probe-based arms use the same budget ($B=40$, $M=8$).

\paragraph{Ablation arms.}
\emph{Random placement} samples eligible interior token positions uniformly instead of using the structured selector, while keeping the probe budget and credit estimator fixed.
\emph{No shrinkage} keeps the structured selector and probe budget fixed but uses the unregularized continuation mean.

\paragraph{Ablation results.}
Table~\ref{tab:qwen-component-ablation} shows that, at the matched checkpoint, replacing structured placement with random placement lowers MATH-500 by 2.29 points and Avg-6 by 1.30; removing shrinkage lowers them by 3.17 and 1.24 points, respectively.
Removing shrinkage also raises the mean per-update 99th percentile of $|A_t|$ by 19.6\%, supporting its regularization role (Appendix~\ref{app:qwen-controls}).
These complementary losses support the central design: sparse measurements are most effective when allocated to informative states and regularized before credit redistribution.

\subsection{Training efficiency}
\label{sec:matched-time}

\paragraph{Matched wall-clock time.}
Using the same learning rate, KL coefficient, and eight-B200 allocation, we compare \method{} and GRPO at matched cumulative training time, including backbone and auxiliary generation, optimization, evaluation, and checkpointing.
Figure~\ref{fig:qwen-matched-time} traces MATH-500 accuracy over cumulative training time.

\begingroup
\setlength{\intextsep}{2pt}
\begin{figure}[H]
\centering
\begin{minipage}[c]{0.53\linewidth}
  \vspace{0pt}
  \centering
  \small
  \setlength{\tabcolsep}{2pt}
  \renewcommand{\arraystretch}{1.10}
  \begin{tabular*}{\linewidth}{@{\extracolsep{\fill}}lcc@{}}
    \toprule
    Variant & MATH-500 & Avg-6 \\
    \midrule
    \method{} & \textbf{79.35} & \textbf{41.36} \\
    Random placement & 77.06 & 40.06 \\
    No shrinkage & 76.18 & 40.12 \\
    \bottomrule
  \end{tabular*}
  \sidecaption{table}{\textbf{Qwen3-4B ablations} (\%).}
  \label{tab:qwen-component-ablation}
\end{minipage}\hfill
\begin{minipage}[c]{0.45\linewidth}
  \vspace{0pt}
  \centering
  \includegraphics[width=\linewidth]{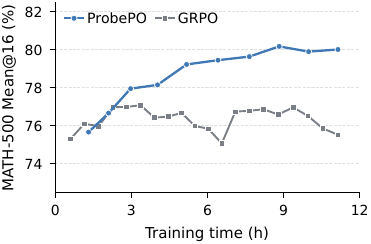}
  \caption{\textbf{Qwen3-4B MATH-500 accuracy at matched training time.}}
  \label{fig:qwen-matched-time}
\end{minipage}
\end{figure}
\endgroup

At 11.14 training hours, \method{} outperforms GRPO by 4.49 points on
MATH-500 and 1.95 points on Avg-6.
Sparse, structured measurements therefore improve accuracy per unit of end-to-end training time.
Endpoint and timing details appear in Appendix~\ref{app:qwen-equal-time}.

\paragraph{Sparse versus dense probing.}
\label{sec:vineppo-efficiency}

VinePPO is the closest dense policy-matched comparator: it estimates values at every eligible reasoning boundary \citep{kazemnejad2025vineppo}.
Under the shared DeepSeekMath-7B-SFT/MATH setup, \method{} exceeds VinePPO on all six benchmarks (19.34 versus 17.31 Avg-6) while using 76\% less training time; VinePPO uses 10.7$\times$ as many auxiliary continuations per update.
Appendix~\ref{app:vineppo-full} reports the per-benchmark results and training-cost details.

\subsection{Offline selector validation}
\label{sec:frontier}
\label{sec:receipts}

We evaluate the reward-free selector post hoc on archived DeepSeekMath-7B-SFT
trajectories using two complementary criteria. At each candidate boundary, we
estimate action-fork variance by sampling several next-token actions and
estimating each action's mean return from repeated continuations. Forward
entropy achieves Spearman correlations of $0.271$ on MATH-precalc and $0.268$
on GSM8K with estimated action-fork variance, with prompt-level bootstrap
intervals excluding zero on both (Figure~\ref{fig:oracle}a). Under a
one-boundary budget, its top-ranked boundary captures $18.0$ and $10.9$
percentage points more within-trajectory variance than random selection,
respectively (Figure~\ref{fig:oracle}b). The random-placement ablation in
Section~\ref{sec:component-ablations} provides complementary downstream
evidence: replacing structured placement with uniform interior-token placement
at the same probe budget lowers MATH-500 accuracy by $2.29$ points.
Appendix~\ref{app:oracle} provides the full protocol and uncertainty scope.

\section{Conclusion}

\method{} allocates sparse measurements to reasoning boundaries and redistributes GRPO credit while preserving each trajectory's mean token-level advantage. Across three models and six benchmarks, it consistently outperforms GRPO, while controlled ablations support structured placement and shrinkage. At matched training time, \method{} improves Avg-6 over GRPO by 1.95 points; against VinePPO, it improves Avg-6 by 2.03 points with 76\% less wall-clock training time. Together, these results show that budgeted, measurement-only probing enables efficient fine-grained credit without dense process supervision.

\bibliographystyle{plainnat}
\bibliography{references}

@misc{qwen3_2025,
  title={{Qwen3} Technical Report},
  author={An Yang and Anfeng Li and Baosong Yang and Beichen Zhang and Binyuan Hui and Bo Zheng and Bowen Yu and Chang Gao and Chengen Huang and Chenxu Lv and Chujie Zheng and Dayiheng Liu and Fan Zhou and Fei Huang and Feng Hu and Hao Ge and Haoran Wei and Huan Lin and Jialong Tang and Jian Yang and Jianhong Tu and Jianwei Zhang and Jianxin Yang and Jiaxi Yang and Jing Zhou and Jingren Zhou and Junyang Lin and Kai Dang and Keqin Bao and Kexin Yang and Le Yu and Lianghao Deng and Mei Li and Mingfeng Xue and Mingze Li and Pei Zhang and Peng Wang and Qin Zhu and Rui Men and Ruize Gao and Shixuan Liu and Shuang Luo and Tianhao Li and Tianyi Tang and Wenbiao Yin and Xingzhang Ren and Xinyu Wang and Xinyu Zhang and Xuancheng Ren and Yang Fan and Yang Su and Yichang Zhang and Yinger Zhang and Yu Wan and Yuqiong Liu and Zekun Wang and Zeyu Cui and Zhenru Zhang and Zhipeng Zhou and Zihan Qiu},
  year={2025},
  eprint={2505.09388},
  archivePrefix={arXiv},
  url={https://arxiv.org/abs/2505.09388},
}

@misc{luo2024omegaprm,
  title={Improve Mathematical Reasoning in Language Models by Automated Process Supervision},
  author={Liangchen Luo and Yinxiao Liu and Rosanne Liu and Samrat Phatale and Meiqi Guo and Harsh Lara and Yunxuan Li and Lei Shu and Yun Zhu and Lei Meng and Jiao Sun and Abhinav Rastogi},
  year={2024},
  eprint={2406.06592},
  archivePrefix={arXiv},
  url={https://arxiv.org/abs/2406.06592},
}

@inproceedings{zheng2025greso,
  title={Act Only When It Pays: Efficient Reinforcement Learning for {LLM} Reasoning via Selective Rollouts},
  author={Zheng, Haizhong and Zhou, Yang and Bartoldson, Brian and Kailkhura, Bhavya and Lai, Fan and Zhao, Jiawei and Chen, Beidi},
  booktitle={Advances in Neural Information Processing Systems},
  volume={38, Main Conference},
  pages={124321--124346},
  publisher={Curran Associates, Inc.},
  year={2025},
  doi={10.52202/085713-4143},
  eprint={2506.02177},
  archivePrefix={arXiv},
  url={https://proceedings.neurips.cc/paper_files/paper/2025/hash/b42fb82fc3ffb88b872c4714f093875d-Abstract-Conference.html},
}

@inproceedings{kazemnejad2025vineppo,
  title={{VinePPO}: Refining Credit Assignment in {RL} Training of {LLM}s},
  author={Kazemnejad, Amirhossein and Aghajohari, Milad and Portelance, Eva and Sordoni, Alessandro and Reddy, Siva and Courville, Aaron and Le Roux, Nicolas},
  booktitle={Proceedings of the 42nd International Conference on Machine Learning},
  pages={29557--29590},
  year={2025},
  volume={267},
  series={Proceedings of Machine Learning Research},
  publisher={PMLR},
  eprint={2410.01679},
  archivePrefix={arXiv},
  url={https://proceedings.mlr.press/v267/kazemnejad25a.html},
}

@inproceedings{hendrycks2021measuring,
  title={Measuring Mathematical Problem Solving With the {MATH} Dataset},
  author={Dan Hendrycks and Collin Burns and Saurav Kadavath and Akul Arora and Steven Basart and Eric Tang and Dawn Song and Jacob Steinhardt},
  booktitle={Proceedings of the Neural Information Processing Systems Track on Datasets and Benchmarks},
  volume={1},
  year={2021},
  url={https://datasets-benchmarks-proceedings.neurips.cc/paper/2021/hash/be83ab3ecd0db773eb2dc1b0a17836a1-Abstract-round2.html},
}

@inproceedings{lightman2024lets,
  title={Let's verify step by step},
  author={Lightman, Hunter and Kosaraju, Vineet and Burda, Yuri and Edwards, Harrison and Baker, Bowen and Lee, Teddy and Leike, Jan and Schulman, John and Sutskever, Ilya and Cobbe, Karl},
  booktitle={International Conference on Learning Representations},
  year={2024}
}

@misc{schulman2017ppo,
  title={Proximal Policy Optimization Algorithms},
  author={John Schulman and Filip Wolski and Prafulla Dhariwal and Alec Radford and Oleg Klimov},
  year={2017},
  eprint={1707.06347},
  archivePrefix={arXiv},
  url={https://arxiv.org/abs/1707.06347},
}

@misc{shao2024deepseekmath,
  title={{DeepSeekMath}: Pushing the Limits of Mathematical Reasoning in Open Language Models},
  author={Zhihong Shao and Peiyi Wang and Qihao Zhu and Runxin Xu and Junxiao Song and Xiao Bi and Haowei Zhang and Mingchuan Zhang and Y. K. Li and Y. Wu and Daya Guo},
  year={2024},
  eprint={2402.03300},
  archivePrefix={arXiv},
  url={https://arxiv.org/abs/2402.03300},
}

@misc{liu2025drgrpo,
  title={Understanding {R1-Zero-Like} Training: A Critical Perspective},
  author={Zichen Liu and Changyu Chen and Wenjun Li and Penghui Qi and Tianyu Pang and Chao Du and Wee Sun Lee and Min Lin},
  year={2025},
  eprint={2503.20783},
  archivePrefix={arXiv},
  url={https://arxiv.org/abs/2503.20783},
}

@misc{guo2025deepseekr1,
  title={{DeepSeek-R1}: Incentivizing Reasoning Capability in {LLM}s via Reinforcement Learning},
  author={ DeepSeek-AI and Daya Guo and Dejian Yang and Haowei Zhang and Junxiao Song and Peiyi Wang and Qihao Zhu and Runxin Xu and Ruoyu Zhang and Shirong Ma and Xiao Bi and Xiaokang Zhang and Xingkai Yu and Yu Wu and Z. F. Wu and Zhibin Gou and Zhihong Shao and Zhuoshu Li and Ziyi Gao and Aixin Liu and Bing Xue and Bingxuan Wang and Bochao Wu and Bei Feng and Chengda Lu and Chenggang Zhao and Chengqi Deng and Chenyu Zhang and Chong Ruan and Damai Dai and Deli Chen and Dongjie Ji and Erhang Li and Fangyun Lin and Fucong Dai and Fuli Luo and Guangbo Hao and Guanting Chen and Guowei Li and H. Zhang and Han Bao and Hanwei Xu and Haocheng Wang and Honghui Ding and Huajian Xin and Huazuo Gao and Hui Qu and Hui Li and Jianzhong Guo and Jiashi Li and Jiawei Wang and Jingchang Chen and Jingyang Yuan and Junjie Qiu and Junlong Li and J. L. Cai and Jiaqi Ni and Jian Liang and Jin Chen and Kai Dong and Kai Hu and Kaige Gao and Kang Guan and Kexin Huang and Kuai Yu and Lean Wang and Lecong Zhang and Liang Zhao and Litong Wang and Liyue Zhang and Lei Xu and Leyi Xia and Mingchuan Zhang and Minghua Zhang and Minghui Tang and Meng Li and Miaojun Wang and Mingming Li and Ning Tian and Panpan Huang and Peng Zhang and Qiancheng Wang and Qinyu Chen and Qiushi Du and Ruiqi Ge and Ruisong Zhang and Ruizhe Pan and Runji Wang and R. J. Chen and R. L. Jin and Ruyi Chen and Shanghao Lu and Shangyan Zhou and Shanhuang Chen and Shengfeng Ye and Shiyu Wang and Shuiping Yu and Shunfeng Zhou and Shuting Pan and S. S. Li and Shuang Zhou and Shaoqing Wu and Shengfeng Ye and Tao Yun and Tian Pei and Tianyu Sun and T. Wang and Wangding Zeng and Wanjia Zhao and Wen Liu and Wenfeng Liang and Wenjun Gao and Wenqin Yu and Wentao Zhang and W. L. Xiao and Wei An and Xiaodong Liu and Xiaohan Wang and Xiaokang Chen and Xiaotao Nie and Xin Cheng and Xin Liu and Xin Xie and Xingchao Liu and Xinyu Yang and Xinyuan Li and Xuecheng Su and Xuheng Lin and X. Q. Li and Xiangyue Jin and Xiaojin Shen and Xiaosha Chen and Xiaowen Sun and Xiaoxiang Wang and Xinnan Song and Xinyi Zhou and Xianzu Wang and Xinxia Shan and Y. K. Li and Y. Q. Wang and Y. X. Wei and Yang Zhang and Yanhong Xu and Yao Li and Yao Zhao and Yaofeng Sun and Yaohui Wang and Yi Yu and Yichao Zhang and Yifan Shi and Yiliang Xiong and Ying He and Yishi Piao and Yisong Wang and Yixuan Tan and Yiyang Ma and Yiyuan Liu and Yongqiang Guo and Yuan Ou and Yuduan Wang and Yue Gong and Yuheng Zou and Yujia He and Yunfan Xiong and Yuxiang Luo and Yuxiang You and Yuxuan Liu and Yuyang Zhou and Y. X. Zhu and Yanhong Xu and Yanping Huang and Yaohui Li and Yi Zheng and Yuchen Zhu and Yunxian Ma and Ying Tang and Yukun Zha and Yuting Yan and Z. Z. Ren and Zehui Ren and Zhangli Sha and Zhe Fu and Zhean Xu and Zhenda Xie and Zhengyan Zhang and Zhewen Hao and Zhicheng Ma and Zhigang Yan and Zhiyu Wu and Zihui Gu and Zijia Zhu and Zijun Liu and Zilin Li and Ziwei Xie and Ziyang Song and Zizheng Pan and Zhen Huang and Zhipeng Xu and Zhongyu Zhang and Zhen Zhang},
  year={2025},
  eprint={2501.12948},
  archivePrefix={arXiv},
  url={https://arxiv.org/abs/2501.12948},
}

@inproceedings{yu2025dapo,
  title={{DAPO}: An Open-Source {LLM} Reinforcement Learning System at Scale},
  author={Yu, Qiying and Zhang, Zheng and Zhu, Ruofei and Yuan, Yufeng and Zuo, Xiaochen and Yue, Yu and Dai, Weinan and Fan, Tiantian and Liu, Gaohong and Liu, Juncai and Liu, LingJun and Liu, Xin and Lin, Haibin and Lin, Zhiqi and Ma, Bole and Sheng, Guangming and Tong, Yuxuan and Zhang, Chi and Zhang, Mofan and Zhang, Ru and Zhang, Wang and Zhu, Hang and Zhu, Jinhua and Chen, Jiaze and Chen, Jiangjie and Wang, Chengyi and Yu, Hongli and Song, Yuxuan and Wei, Xiangpeng and Zhou, Hao and Liu, Jingjing and Ma, Wei-Ying and Zhang, Ya-Qin and Yan, Lin and Wu, Yonghui and Wang, Mingxuan},
  booktitle={Advances in Neural Information Processing Systems},
  volume={38, Main Conference},
  pages={113222--113244},
  publisher={Curran Associates, Inc.},
  year={2025},
  doi={10.52202/085713-3775},
  eprint={2503.14476},
  archivePrefix={arXiv},
  url={https://proceedings.neurips.cc/paper_files/paper/2025/hash/a4277440d50f1f15d2cb4c14f7e0c0d2-Abstract-Conference.html},
}

@inproceedings{li2025treepo,
  title={{TreePO}: Enhancing Policy Efficacy and Inference Efficiency with Tree Modeling},
  author={Li, Yizhi and Gu, Qingshui and Wen, Zhoufutu and Li, Ziniu and Yuan, Ruibin and Xing, Tianshun and Guo, Shuyue and Zheng, Tuney and Zhou, Xin and Qu, Xingwei and Zhou, Wangchunshu and Zhang, Zheng and Shen, Wei and Xue, Wei and Liu, Qian and Lin, Chenghua and Yang, Jian and Zhang, Ge and Huang, Wenhao},
  booktitle={Forty-Third International Conference on Machine Learning},
  year={2026},
  eprint={2508.17445},
  archivePrefix={arXiv},
  url={https://icml.cc/virtual/2026/poster/61677},
}

@misc{yang2025treerpo,
  title={{TreeRPO}: Tree Relative Policy Optimization},
  author={Zhicheng Yang and Zhijiang Guo and Yinya Huang and Xiaodan Liang and Yiwei Wang and Jing Tang},
  year={2025},
  eprint={2506.05183},
  archivePrefix={arXiv},
  url={https://arxiv.org/abs/2506.05183},
}

@inproceedings{guo2025spo,
  title={Segment Policy Optimization: Effective Segment-Level Credit Assignment in {RL} for Large Language Models},
  author={Guo, Yiran and Xu, Lijie and Liu, Jie and Ye, Dan and Qiu, Shuang},
  booktitle={Advances in Neural Information Processing Systems},
  volume={38, Main Conference},
  pages={114399--114431},
  publisher={Curran Associates, Inc.},
  year={2025},
  doi={10.52202/085713-3815},
  eprint={2505.23564},
  archivePrefix={arXiv},
  url={https://proceedings.neurips.cc/paper_files/paper/2025/hash/a6536243037d1e32c20de85137d478da-Abstract-Conference.html},
}

@inproceedings{wang2024mathshepherd,
  title={{Math-Shepherd}: Verify and Reinforce {LLM}s Step-by-step without Human Annotations},
  author={Wang, Peiyi and Li, Lei and Shao, Zhihong and Xu, Runxin and Dai, Damai and Li, Yifei and Chen, Deli and Wu, Yu and Sui, Zhifang},
  booktitle={Proceedings of the 62nd Annual Meeting of the Association for Computational Linguistics (Volume 1: Long Papers)},
  month={aug},
  year={2024},
  address={Bangkok, Thailand},
  publisher={Association for Computational Linguistics},
  pages={9426--9439},
  doi={10.18653/v1/2024.acl-long.510},
  eprint={2312.08935},
  archivePrefix={arXiv},
  url={https://aclanthology.org/2024.acl-long.510/},
}

@article{neyman1934,
  title={On the Two Different Aspects of the Representative Method: The Method of Stratified Sampling and the Method of Purposive Selection},
  author={Neyman, Jerzy},
  journal={Journal of the Royal Statistical Society},
  volume={97},
  number={4},
  pages={558--625},
  year={1934},
  doi={10.2307/2342192},
  url={https://doi.org/10.2307/2342192},
}

@book{gelman2013bayesian,
  title={Bayesian Data Analysis},
  author={Gelman, Andrew and Carlin, John B. and Stern, Hal S. and Dunson, David B. and Vehtari, Aki and Rubin, Donald B.},
  edition={3},
  publisher={Chapman and Hall/CRC},
  address={Boca Raton, FL},
  year={2013},
  isbn={9780429113079},
  doi={10.1201/b16018},
  url={https://doi.org/10.1201/b16018},
}

@article{chujie2025group,
  title={Group Sequence Policy Optimization},
  author={Chujie Zheng and Shixuan Liu and Mingze Li and Xiong-Hui Chen and Bowen Yu and Chang Gao and Kai Dang and Yuqiong Liu and Rui Men and An Yang and Jingren Zhou and Junyang Lin},
  journal={arXiv preprint arXiv:2507.18071},
  eprint={2507.18071},
  archivePrefix={arXiv},
  url={https://arxiv.org/abs/2507.18071},
  year={2025}
}

@article{haldar2026f,
  title={{f-GRPO} and Beyond: Divergence-Based Reinforcement Learning Algorithms for General {LLM} Alignment},
  author={Haldar, Rajdeep and Mei, Lantao and Lin, Guang and Xing, Yue and Song, Qifan},
  journal={arXiv preprint arXiv:2602.05946},
  year={2026},
  eprint={2602.05946},
  archivePrefix={arXiv},
  url={https://arxiv.org/abs/2602.05946},
}

@article{hatamizadeh2026igrpo,
  title={{iGRPO}: Self-Feedback-Driven {LLM} Reasoning},
  author={Hatamizadeh, Ali and Prabhumoye, Shrimai and Gitman, Igor and Lu, Ximing and Han, Seungju and Ping, Wei and Choi, Yejin and Kautz, Jan},
  journal={arXiv preprint arXiv:2602.09000},
  year={2026},
  eprint={2602.09000},
  archivePrefix={arXiv},
  url={https://arxiv.org/abs/2602.09000},
}

@article{yoo2026rollout,
  title={Rollout-Level Advantage-Prioritized Experience Replay for {GRPO}},
  author={Yoo, Gyeongtae and Park, Sanghyeok and Jang, Soohyuk and Kim, Ik-hwan and Yoon, Sungroh},
  journal={arXiv preprint arXiv:2606.04560},
  year={2026},
  eprint={2606.04560},
  archivePrefix={arXiv},
  url={https://arxiv.org/abs/2606.04560},
}

@article{wang2025lambda,
  title={{$\lambda$-GRPO}: Unifying the {GRPO} Frameworks with Learnable Token Preferences},
  author={Wang, Yining and Zhao, Jinman and Zhao, Chuangxin and Guan, Shuhao and Penn, Gerald and Liu, Shinan},
  journal={arXiv preprint arXiv:2510.06870},
  year={2025},
  eprint={2510.06870},
  archivePrefix={arXiv},
  url={https://arxiv.org/abs/2510.06870},
}

@inproceedings{setlur2025rewarding,
  title={Rewarding Progress: Scaling Automated Process Verifiers for {LLM} Reasoning},
  author={Setlur, Amrith and Nagpal, Chirag and Fisch, Adam and Geng, Xinyang and Eisenstein, Jacob and Agarwal, Rishabh and Agarwal, Alekh and Berant, Jonathan and Kumar, Aviral},
  booktitle={International Conference on Learning Representations},
  volume={2025},
  pages={60808--60838},
  year={2025},
  eprint={2410.08146},
  archivePrefix={arXiv},
  url={https://proceedings.iclr.cc/paper_files/paper/2025/hash/98711dea460bdefe0e651ca23ec98ba2-Abstract-Conference.html},
}

@inproceedings{li2025red,
  title={{RED}: Unleashing Token-Level Rewards from Holistic Feedback via Reward Redistribution},
  author={Li, Jiahui and Li, Lin and Chang, Tai-Wei and Kuang, Kun and Chen, Long and Zhou, Jun and Yang, Cheng},
  booktitle={Proceedings of the 2025 Conference on Empirical Methods in Natural Language Processing},
  month={nov},
  address={Suzhou, China},
  publisher={Association for Computational Linguistics},
  pages={4993--5022},
  year={2025},
  doi={10.18653/v1/2025.emnlp-main.252},
  eprint={2411.08302},
  archivePrefix={arXiv},
  url={https://aclanthology.org/2025.emnlp-main.252/},
}

@inproceedings{zhou2025t,
  title={{T}-{REG}: Preference Optimization with Token-Level Reward Regularization},
  author={Zhou, Wenxuan and Zhang, Shujian and Zhao, Lingxiao and Meng, Tao},
  booktitle={Proceedings of the 63rd Annual Meeting of the Association for Computational Linguistics (Volume 1: Long Papers)},
  month={jul},
  address={Vienna, Austria},
  publisher={Association for Computational Linguistics},
  pages={27876--27889},
  year={2025},
  doi={10.18653/v1/2025.acl-long.1353},
  eprint={2412.02685},
  archivePrefix={arXiv},
  url={https://aclanthology.org/2025.acl-long.1353/},
}

@article{cao2025scar,
  title={{SCAR}: Shapley Credit Assignment for More Efficient {RLHF}},
  author={Cao, Meng and Zhang, Shuyuan and Chang, Xiao-Wen and Precup, Doina},
  journal={arXiv preprint arXiv:2505.20417},
  year={2025},
  eprint={2505.20417},
  archivePrefix={arXiv},
  url={https://arxiv.org/abs/2505.20417},
}

@inproceedings{yang2026int,
  title={{InT}: Self-Proposed Interventions Enable Credit Assignment in {LLM} Reasoning},
  author={Yang, Matthew and Bai, Hao and Wu, Ian and Yang, Gene and Setlur, Amrith and Kumar, Aviral},
  booktitle={International Conference on Learning Representations},
  volume={2026},
  pages={85054--85091},
  year={2026},
  eprint={2601.14209},
  archivePrefix={arXiv},
  url={https://proceedings.iclr.cc/paper_files/paper/2026/hash/89062e4d480c0c3a88d36c20c5694459-Abstract-Conference.html},
}

@article{xie2025capo,
  title={{CAPO}: Towards Enhancing {LLM} Reasoning through Generative Credit Assignment},
  author={Xie, Guofu and Shi, Yunsheng and Tian, Hongtao and Yao, Ting and Zhang, Xiao},
  journal={arXiv preprint arXiv:2508.02298},
  year={2025},
  eprint={2508.02298},
  archivePrefix={arXiv},
  url={https://arxiv.org/abs/2508.02298},
}

@inproceedings{li2026ssvpo,
  title={SSVPO: Effective Step-Level Credit Assignment for RL Training of Language Models},
  author={Li, Yugu and Cao, Zehong and Qiao, Jianglin and Hu, Siyi},
  booktitle={The Fourteenth International Conference on Learning Representations},
  year={2026}
}

@article{wang2026grpo,
  title={{GRPO-VPS}: Enhancing Group Relative Policy Optimization with Verifiable Process Supervision for Effective Reasoning},
  author={Wang, Jingyi and Zhu, Lei and Weng, Tengjin and Wu, Song-Li and Tan, Haochen and Chen, Jierun and Tao, Chaofan and Bai, Haoli and Hou, Lu and Shang, Lifeng and Zhang, Xiao-Ping},
  journal={arXiv preprint arXiv:2604.20659},
  year={2026},
  eprint={2604.20659},
  archivePrefix={arXiv},
  url={https://arxiv.org/abs/2604.20659},
}

@article{tran2025exploiting,
  title={Exploiting Tree Structure for Credit Assignment in {RL} Training of {LLM}s},
  author={Tran, Hieu and Yao, Zonghai and Yu, Hong},
  journal={arXiv preprint arXiv:2509.18314},
  year={2025},
  eprint={2509.18314},
  archivePrefix={arXiv},
  url={https://arxiv.org/abs/2509.18314},
}

@inproceedings{zhao2026reinforced,
  title={Reinforced Efficient Reasoning via Semantically Diverse Exploration},
  author={Zhao, Ziqi and Ren, Zhaochun and Zou, Jiahong and Yang, Liu and Xu, Zhiwei and Ge, Xuri and Chen, Zhumin and Ma, Xinyu and Shi, Daiting and Wang, Shuaiqiang and Yin, Dawei and Xin, Xin},
  booktitle={Proceedings of the 64th Annual Meeting of the Association for Computational Linguistics (Volume 1: Long Papers)},
  month={jul},
  address={San Diego, California, United States},
  publisher={Association for Computational Linguistics},
  pages={47994--48007},
  year={2026},
  doi={10.18653/v1/2026.acl-long.2216},
  eprint={2601.05053},
  archivePrefix={arXiv},
  url={https://aclanthology.org/2026.acl-long.2216/},
}

\clearpage
\appendix
\raggedbottom

\section{A conditional precision-allocation benchmark}
\label{app:allocation}

\method{} does not optimize the number of continuations per selected pivot: it uses $n_b=M$ throughout.
This appendix records a benchmark for a different, conditional question---if the pivot set were already fixed and per-pivot sample counts could vary, how would a first-order approximation allocate a total budget $B$ to reduce MC-induced gradient perturbation?
The benchmark is not used by the algorithm and does not derive the $H_{\max}$ placement score.

\paragraph{Setup.}
Condition on a fixed accepted backbone trajectory and a fixed pivot set $\mathcal P$.
To isolate MC precision, let the raw, unshrunk mean at $b\in\mathcal P$ use $n_b$ continuations, with error $\epsilon_b=\Vhat_{\mathrm{raw}}(s_b)-V^{\pi_\theta}(s_b)$, $\E[\epsilon_b\mid s_b]=0$, and $\Var(\epsilon_b\mid s_b)=\Var(R\mid s_b)/n_b$.
For this appendix only, let $A^{(0)}_{i,t}$ denote Eq.~\eqref{eq:final} with the raw pivot means inserted directly (equivalently, $n_0=0$), and define $J_{tb}:=\partial A^{(0)}_{i,t}/\partial\Vhat_{\mathrm{raw}}(s_b)$.
Writing $\psi_t=\nabla_\theta\log\pi_\theta(y_t\mid s_{t-1})$, the probe-induced gradient perturbation is
$\delta g=\sum_t\big(\sum_{b\in\mathcal P}J_{tb}\epsilon_b\big)\psi_t$.
For this unclipped gradient proxy and conditionally independent probe sets,
\[
\E_{\epsilon}\|\delta g\|^2
\;=\;
\sum_{b\in\mathcal P}\frac{\Var(R\mid s_b)}{n_b}\,C_b,
\qquad
C_b=\Big\|\sum_t J_{tb}\psi_t\Big\|^2.
\]
Assuming one continuation per budget unit and relaxing $n_b$ to positive reals, minimizing this proxy under $\sum_{b\in\mathcal P}n_b=B$ gives the Neyman-style allocation \citep{neyman1934}
\begin{equation}
n_b^\star
=B\,\frac{\sqrt{C_b\,\Var(R\mid s_b)}}
{\sum_{b'\in\mathcal P}\sqrt{C_{b'}\,\Var(R\mid s_{b'})}}
\;\propto\;
\sqrt{C_b\,\Var(R\mid s_b)}.
\label{eq:neyman}
\end{equation}

\paragraph{Scope and approximations.}
First, Eq.~\eqref{eq:neyman} allocates precision among an already selected pivot set; it neither selects pivots nor supplies an analytic justification for $H_{\max}$.
Second, before centering, a boundary estimate enters only the two adjacent segments; after Eq.~\eqref{eq:center}, it also changes the trajectory-wide centering offset.
For an interior boundary $b=\tau_j$, let $I_b^-=I_j$ denote the segment ending at $b$ and $I_b^+=I_{j+1}$ the segment beginning after $b$.
Under the unshrunk map, the corresponding final advantage sensitivity is
\[
J_{tb}=\frac{\alpha}{\sigma_g}\left[
\mathbf{1}\{t\in I_b^-\}-\mathbf{1}\{t\in I_b^+\}
-\frac{|I_b^-|-|I_b^+|}{T_i}
\right].
\]
Thus $C_b$ must use this full centered Jacobian; adjacent-segment length alone is not generally the correct footprint.
Third, the benchmark ignores policy-ratio clipping and optimizer dynamics, ignores cross-pivot covariance, treats $n_b$ as continuous, and assigns equal cost to continuations despite variation in generated suffix length.
Fourth, it omits the contraction in Eq.~\eqref{eq:shrink}, which introduces bias and changes mean-squared error.
Finally, Eq.~\eqref{eq:neyman} depends on total continuation variance $\Var(R\mid s_b)$, whereas placement targets action-fork variance $\Var_{a\sim\pi_\theta}(Q^{\pi_\theta}(s_b,a))$.
For all of these reasons, we retain Eq.~\eqref{eq:neyman} only as a precision benchmark; the implemented method ranks with $H_{\max}$ and assigns the same $M$ continuations to every selected pivot.

\section{Estimator details and token-mean preservation}
\label{app:proofs}

\begin{proof}[Proof of Proposition~\ref{prop:mean}]
Equation~\eqref{eq:center} gives $\sum_t\tilde A_{i,t}=0$; hence Eq.~\eqref{eq:final} yields
$T_i^{-1}\sum_t A_{i,t}=A_{\mathrm{base},i}=(r_i-\mu_g)/\sigma_g$, proving the first claim.
If $K=0$, $A^{\mathrm{raw}}_{i,t}=d_1/\sigma_g$ is constant across the response, so its centered residual vanishes and
$A_{i,t}=A_{\mathrm{base},i}=(r_i-\mu_g)/\sigma_g$.
\end{proof}

\paragraph{Pseudo-count interpretation and bias--variance decomposition.}
For binary rewards, let $S_k=M\Vhat_{\mathrm{raw}}(\tau_k)$ and $v_k=V^{\pi_\theta}(s_{\tau_k})$.
When $n_0>0$ and $0<\mu_g<1$, the working prior
$v_k\sim\operatorname{Beta}(n_0\mu_g,n_0(1-\mu_g))$ yields Eq.~\eqref{eq:shrink} as the posterior mean after observing $S_k$ successes in $M$ trials.
Because the prior center $\mu_g$ is estimated from the accepted backbone group, this is a plug-in empirical-Bayes interpretation rather than a fully specified generative model.

More generally, set $\lambda=M/(M+n_0)$ and $\varsigma_k^2=\Var(R\mid s_{\tau_k})$.
Conditional on the selected state and the realized target $\mu_g$, independent on-policy probes give
\[
\mathbb E[\Vhat(\tau_k)]-v_k
=(1-\lambda)(\mu_g-v_k),
\qquad
\Var(\Vhat(\tau_k))
=\lambda^2\frac{\varsigma_k^2}{M},
\]
and hence
\[
\operatorname{MSE}(\Vhat(\tau_k))
=
\lambda^2\frac{\varsigma_k^2}{M}
+(1-\lambda)^2(v_k-\mu_g)^2.
\]
Thus $n_0$ makes explicit the tradeoff between Monte Carlo variance and mismatch between the pivot value and the group-level target.

\paragraph{Exact-update conditions.}
On the same accepted backbone trajectories, $K=0$ or $\alpha=0$ reproduces the configured GRPO advantage tensor when group statistics, response masks, and standardization are identical.
It also reproduces the optimization update when the policy parameters, objective (including clipping and KL terms), and optimizer state and settings are matched.

\paragraph{Token-mean advantages}
For the unclipped score-function estimator at the rollout policy, the residual adds
\[
\Delta g_i
=
\frac{\alpha}{T_i}
\sum_{t=1}^{T_i}
\tilde A_{i,t}
\nabla_\theta\log\pi_\theta(y_{i,t}\mid q_i,y_{i,<t}).
\]
Although $\sum_t\tilde A_{i,t}=0$, this score-weighted sum is generally nonzero because the token score vectors differ.
Proposition~\ref{prop:mean} therefore preserves each trajectory's token mean advantage, not the update produced by a nonzero residual.

\paragraph{Centering and segment length.}
Because the centering subtracts the token-mean, a residual's effect on a segment scales with the complement of that segment's length share; long segments retain proportionally less of their raw delta than short ones.

\paragraph{Shrinkage is not a fallback.}
Suppose $K\ge1$ and every interior estimate is contracted fully to $\mu_g$.
Writing $p=(T_i-\tau_K)/T_i$ for the terminal segment's token fraction and $b_i=(r_i-\mu_g)/\sigma_g$, the raw token-level advantage is $0$ before the terminal segment and $b_i$ on it.
After centering, the residual is therefore $-p b_i$ before that segment and $(1-p)b_i$ on it, which is generally nonzero.
Removing all pivots ($K=0$) or setting the residual weight to zero ($\alpha=0$) guarantees the configured GRPO advantage tensor independently of the probe values; contraction alone does not.

\section{Cutoff-rescue probe policy}
\label{app:forced_answer}

The R1-Distill-Qwen-1.5B regime (2K context) uses a cutoff-rescue policy for probes that would otherwise exhaust the context window without a gradeable answer.
From each selected prefix, a probe first continues normally under $\pi_\theta$.
If it terminates before the operational cutoff, its continuation is kept and scored unchanged.
If it remains unfinished at the cutoff, the continuation is trimmed there and scored directly when the trimmed text is already gradeable.
Otherwise, an open reasoning block is closed when necessary, an answer template is appended, and a short reserved decode produces the answer submitted to the verifier.
Thus natural and rescued samples are outcomes of one modified continuation policy $\pi_{\mathrm{rescue},k}$, and their mean estimates $V^{\pi_{\mathrm{rescue},k}}(s)$ rather than the unmodified $V^{\pi_\theta}(s)$.
In the recipe-matched replay, 98.0\% of probe continuations finish naturally; the remaining 2.0\% use cutoff handling.

\section{Extended experimental results}
\label{app:extended-results}

\begin{figure}[H]
  \centering
  \definecolor{FigInk}{HTML}{26364B}
\definecolor{FigGray}{HTML}{758093}
\definecolor{FigBlue}{HTML}{3F78B5}
\definecolor{FigPurple}{HTML}{7655A5}
\definecolor{FigGreen}{HTML}{2D8768}
\definecolor{FigRed}{HTML}{C95252}
\resizebox{\linewidth}{!}{%
\begin{tikzpicture}[
  font=\sffamily\scriptsize,
  line cap=round,
  dotP/.style={circle, fill=FigBlue, inner sep=0pt, minimum size=4.6pt},
  dotT/.style={rectangle, fill=FigGreen, inner sep=0pt, minimum size=4.2pt},
  dotS/.style={diamond, fill=FigPurple, inner sep=0pt, minimum size=4.8pt},
  vlab/.style={font=\sffamily\scriptsize, text=FigInk, inner sep=1.2pt},
  cell/.style={draw=white, line width=0.55pt, minimum width=1.05cm, minimum height=0.50cm, inner sep=0pt, font=\sffamily\scriptsize}
]
\begin{scope}
  \node[anchor=west, font=\sffamily\bfseries\scriptsize, text=FigInk] at (-3.90,0.85)
    {\textbf{(a)}\enspace $\Delta$ Avg-6 vs.\ GRPO reference (pp)};
  \node[dotP] at (-3.80,0.42) {}; \node[anchor=west, text=FigInk] at (-3.70,0.42) {ProbePO};
  \node[dotT] at (-2.35,0.42) {}; \node[anchor=west, text=FigInk] at (-2.25,0.42) {TreePO};
  \node[dotS] at (-0.95,0.42) {}; \node[anchor=west, text=FigInk] at (-0.85,0.42) {SPO-tree};
  \foreach \v in {-4,-2,2,4}
    \draw[FigGray!30, line width=0.4pt] (\v*0.30,-3.30) -- (\v*0.30,0.10);
  \draw[FigInk!70, line width=0.6pt] (0,-3.30) -- (0,0.10);
  \foreach \v in {-4,-2,0,2,4}
    \node[anchor=north, text=FigGray!88!black, font=\sffamily\scriptsize] at (\v*0.30,-3.32) {\v};
  \node[anchor=east, text=FigInk, align=right] at (-1.65,-0.50) {DeepSeekMath-7B};
  \node[dotP] (a1) at (1.72*0.30,-0.20) {}; \node[vlab, anchor=west] at (a1.east) {+1.72};
  \node[dotT] (a2) at (-1.04*0.30,-0.50) {}; \node[vlab, anchor=east] at (a2.west) {$-$1.04};
  \node[dotS] (a3) at (-2.28*0.30,-0.80) {}; \node[vlab, anchor=east] at (a3.west) {$-$2.28};
  \node[anchor=east, text=FigInk, align=right] at (-1.65,-1.55) {R1-Distill-Qwen-1.5B\\(2K context)};
  \node[dotP] (b1) at (1.89*0.30,-1.25) {}; \node[vlab, anchor=west] at (b1.east) {+1.89};
  \node[dotT] (b2) at (0.48*0.30,-1.55) {}; \node[vlab, anchor=west] at (b2.east) {+0.48};
  \node[dotS] (b3) at (-0.48*0.30,-1.85) {}; \node[vlab, anchor=east] at (b3.west) {$-$0.48};
  \node[anchor=east, text=FigInk, align=right] at (-1.65,-2.60) {Qwen3-4B};
  \node[dotP] (d1) at (2.84*0.30,-2.30) {}; \node[vlab, anchor=west] at (d1.east) {+2.84};
  \node[dotT] (d2) at (0.78*0.30,-2.60) {}; \node[vlab, anchor=west] at (d2.east) {+0.78};
  \node[dotS] (d3) at (0.17*0.30,-2.90) {}; \node[vlab, anchor=west] at (d3.east) {+0.17};
\end{scope}
\begin{scope}[shift={(3.80,0)}]
  \node[anchor=west, font=\sffamily\bfseries\scriptsize, text=FigInk] at (-1.10,0.85)
    {\textbf{(b)}\enspace ProbePO $-$ GRPO by benchmark (pp)};
  \node[anchor=south, text=FigInk, font=\sffamily\scriptsize, align=center] at (0*1.05+0.525,-0.55) {MATH-\\500};
  \node[anchor=south, text=FigInk, font=\sffamily\scriptsize, align=center] at (1*1.05+0.525,-0.55) {Olympiad\\Bench};
  \node[anchor=south, text=FigInk, font=\sffamily\scriptsize, align=center] at (2*1.05+0.525,-0.55) {College\\Math};
  \node[anchor=south, text=FigInk, font=\sffamily\scriptsize, align=center] at (3*1.05+0.525,-0.55) {AMC};
  \node[anchor=south, text=FigInk, font=\sffamily\scriptsize, align=center] at (4*1.05+0.525,-0.55) {Minerva\\Math};
  \node[anchor=south, text=FigInk, font=\sffamily\scriptsize, align=center] at (5*1.05+0.28+0.525,-0.55) {AIME24\\(@32)};
  \draw[FigGray!70, line width=0.5pt, densely dotted] (5*1.05+0.14,-0.58) -- (5*1.05+0.14,-3.00);
  \node[anchor=east, text=FigInk] at (-0.08,-1.00) {DeepSeek-7B};
  \node[anchor=east, text=FigInk, align=right] at (-0.08,-1.85) {R1-Distill-Qwen\\1.5B};
  \node[anchor=east, text=FigInk] at (-0.08,-2.70) {Qwen3-4B};
  \node[cell, fill=FigBlue!31]              at (0*1.05+0.525,-1.00) {+3.26};
  \node[cell, fill=FigBlue!21]              at (1*1.05+0.525,-1.00) {+1.81};
  \node[cell, fill=FigBlue!11]              at (2*1.05+0.525,-1.00) {+0.20};
  \node[cell, fill=FigBlue!27]              at (3*1.05+0.525,-1.00) {+2.71};
  \node[cell, fill=FigBlue!17]              at (4*1.05+0.525,-1.00) {+1.17};
  \node[cell, fill=FigBlue!17]              at (5*1.05+0.28+0.525,-1.00) {+1.15};
  \node[cell, fill=FigBlue!24]              at (0*1.05+0.525,-1.85) {+2.24};
  \node[cell, fill=FigBlue!25]              at (1*1.05+0.525,-1.85) {+2.35};
  \node[cell, fill=FigBlue!21]              at (2*1.05+0.525,-1.85) {+1.86};
  \node[cell, fill=FigBlue!31]              at (3*1.05+0.525,-1.85) {+3.24};
  \node[cell, fill=FigBlue!14]              at (4*1.05+0.525,-1.85) {+0.60};
  \node[cell, fill=FigBlue!16]              at (5*1.05+0.28+0.525,-1.85) {+1.04};
  \node[cell, fill=FigBlue!32]              at (0*1.05+0.525,-2.70) {+3.40};
  \node[cell, fill=FigBlue!26]              at (1*1.05+0.525,-2.70) {+2.63};
  \node[cell, fill=FigBlue!12]              at (2*1.05+0.525,-2.70) {+0.39};
  \node[cell, fill=FigBlue!32]              at (3*1.05+0.525,-2.70) {+3.39};
  \node[cell, fill=FigBlue!22]              at (4*1.05+0.525,-2.70) {+1.84};
  \node[cell, fill=FigBlue!42]              at (5*1.05+0.28+0.525,-2.70) {+5.42};
  \draw[FigGray, densely dashed, line width=0.6pt] (4*1.05,-2.10) rectangle (4*1.05+1.05,-1.60);
  \draw[FigGray, densely dashed, line width=0.6pt] (5*1.05+0.28,-2.10) rectangle (5*1.05+0.28+1.05,-1.60);
\end{scope}
\end{tikzpicture}%
}
  \caption{\textbf{Terminal-checkpoint gains.}
  \emph{(a)} Avg-6 difference from each setting's GRPO reference.
  \emph{(b)} \method{} minus GRPO by benchmark.
  AIME24 uses Mean@32; dashed cells mark tasks on which another method scores higher.}
  \label{fig:outcome}
\end{figure}
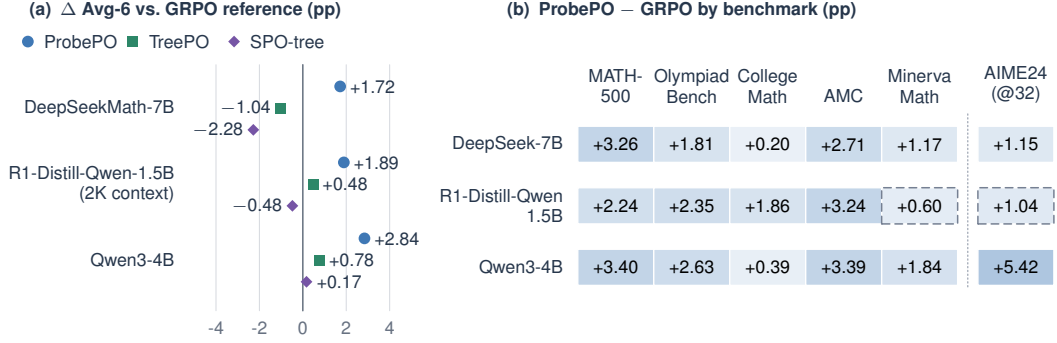

\subsection{Full dense-probing comparison}
\label{app:vineppo-full}

\begin{table}[H]
\centering
\caption{\textbf{Quality and cost against dense probing.}
Panel (a) reports terminal benchmark accuracy under the shared DeepSeekMath
setting; panel (b) reports end-to-end training cost.}
\label{tab:vineppo-quality-full}
\definecolor{tablegroup}{RGB}{244,245,247}
\definecolor{tableours}{RGB}{238,246,241}
\definecolor{tableavg}{RGB}{247,247,247}
\setlength{\tabcolsep}{3pt}
\renewcommand{\arraystretch}{1.05}

{\small
\begin{tabularx}{\linewidth}{@{}>{\raggedright\arraybackslash}p{0.17\linewidth}*{6}{>{\centering\arraybackslash}X}>{\columncolor{tableavg}\centering\arraybackslash}X@{}}
\toprule
\rowcolor{tablegroup}
\multicolumn{8}{@{}l@{}}{\textbf{(a) Terminal accuracy (\%; higher is better)}} \\
\midrule
\textbf{Method} & \makecell[c]{\textbf{MATH}\\\textbf{500}} & \makecell[c]{\textbf{Olympiad}\\\textbf{Bench}} & \makecell[c]{\textbf{College}\\\textbf{Math}} & \textbf{AMC} & \makecell[c]{\textbf{Minerva}\\\textbf{Math}} & \makecell[c]{\textbf{AIME24}\\\textbf{(@32)}} & \textbf{Avg-6} \\
\midrule
\rowcolor{tableours}
\textbf{\method{} (ours)} & \textbf{45.81} & \textbf{14.47} & \textbf{22.77} & \textbf{17.09} & \textbf{12.96} & \textbf{2.92} & \textbf{19.34} \\
VinePPO & 43.69 & 11.65 & 22.11 & 15.81 & 10.27 & 0.31 & 17.31 \\
\bottomrule
\end{tabularx}
}

\vspace{6pt}

{\footnotesize
\begin{tabularx}{\linewidth}{@{}>{\raggedright\arraybackslash}p{0.17\linewidth}>{\centering\arraybackslash}p{0.38\linewidth}>{\centering\arraybackslash}p{0.20\linewidth}>{\centering\arraybackslash}X@{}}
\toprule
\rowcolor{tablegroup}
\multicolumn{4}{@{}l@{}}{\textbf{(b) End-to-end training cost (lower is better)}} \\
\midrule
\textbf{Method} & \makecell[c]{\textbf{Auxiliary continuations}\\\textbf{per update}} & \textbf{H200-hours} & \makecell[c]{\textbf{Relative}\\\textbf{time}} \\
\midrule
\rowcolor{tableours}
\textbf{\method{} (ours)} & $\leq 2{,}560$ & \textbf{231} & \textbf{1.00$\times$} \\
VinePPO & 27{,}265 mean & 944 & 4.08$\times$ \\
\bottomrule
\end{tabularx}

\vspace{3pt}
\parbox{\linewidth}{\raggedright
\emph{Note.} \method{}'s auxiliary-continuation count is a configured upper
bound, whereas VinePPO's is the observed per-update mean. H200-hours are
rounded; relative time uses the measured end-to-end durations.}
}
\end{table}

The original VinePPO paper reports 46.0\% for its MATH-500 entry.  Our
evaluation instead uses the strict verl/Hendrycks grader rather than
VinePPO's Minerva-style grader and runs on eight H200 rather than H100 GPUs.
Both methods start from the same DeepSeekMath-7B-SFT initialization and MATH
source split, use 64 prompt groups with $G=8$ backbone responses, cap responses
at 1{,}024 tokens, and train for 1{,}000 outer updates on one exclusive node.
VinePPO retains its released nine-continuation-per-boundary recipe.

\subsection{Controlled Qwen3 component study}
\label{app:qwen-controls}

The study in Section~\ref{sec:component-ablations} uses Qwen3-4B-Base, the
common ROSE MATH data and objective, the same frozen evaluator, and one
$8\times$B200 node.  Every arm retains all generated prompt groups, and the
probe-based arms share the same normalization, coefficient clipping, and
$B=40$, $M=8$, and $\alpha=1$ settings.

The structured reference uses $H_{\max}$ with $L=8$, a 60-token separation, a
two-pivot core cap per trajectory, and post-answer exclusion.  Random
placement replaces this complete placement module with uniform sampling over
interior pre-EOS token positions while preserving the probe budget and credit
estimator.  It therefore evaluates structured placement as a module rather
than entropy ranking alone.  The no-shrinkage arm changes only $n_0$ from 4 to
0, removing the group-baseline contribution to the value estimate.
Table~\ref{tab:qwen-component-ablation} reports the matched component results.
At each update, we compute the 99th percentile of $|A_t|$ over response tokens
included in the policy loss and average it across the common component-study
training window.  Removing shrinkage raises this statistic from 2.672 to 3.195
(19.6\%), consistent with its regularization role.

\subsection{Accuracy at matched training time}
\label{app:qwen-equal-time}

\begin{table}[H]
\centering
\caption{\textbf{Equal-training-time Qwen3-4B comparison} (\%).
Both methods use approximately 11.14 logged trainer hours on the same
$8\times$B200 hardware.}
\label{tab:qwen-equal-time}
\footnotesize
\renewcommand{\arraystretch}{0.98}
\setlength{\tabcolsep}{8pt}
\begin{tabular}{lccc}
\toprule
Method & Training time (hours) & MATH-500 & Avg-6 \\
\midrule
\method{} & 11.14 & \textbf{80.01} & \textbf{41.03} \\
GRPO & 11.14 & 75.52 & 39.08 \\
\midrule
$\Delta$ & --- & \textbf{$+4.49$} & \textbf{$+1.95$} \\
\bottomrule
\end{tabular}
\end{table}

Cumulative logged trainer time includes backbone and auxiliary generation,
optimization, periodic evaluation, and checkpointing.

\subsection{Headline Qwen3-4B training dynamics}
\label{app:qwen3-dynamics}
Figure~\ref{fig:qwen3-training-dynamics} complements the terminal comparison
with performance over the shared training horizon.  On MATH-500, \method{}
maintains the strongest trajectory throughout the later evaluations.  At the
final monitored checkpoint, it also leads on both AIME benchmarks and remains
competitive on AMC.  Together, these curves show that the headline MATH-500
result reflects sustained training behavior rather than an isolated
checkpoint.

\begin{figure}[H]
  \centering
  \includegraphics[width=\linewidth]{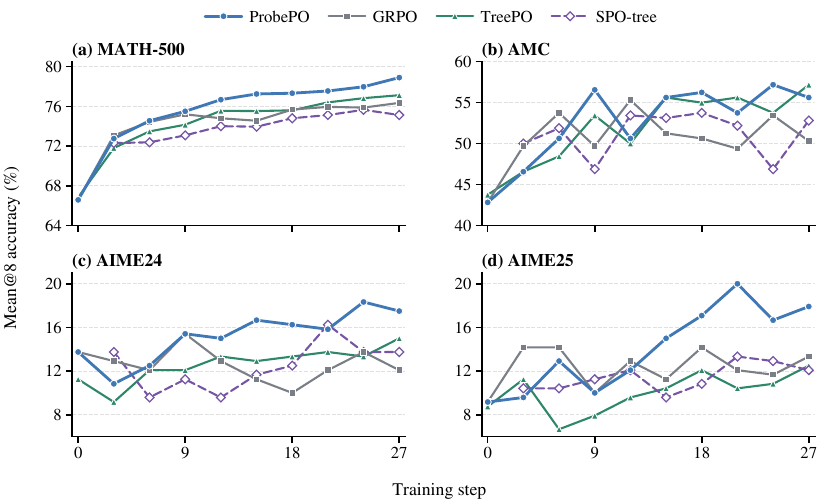}
  \caption{\textbf{Qwen3-4B training dynamics.}
  Monitoring accuracy over the shared headline horizon; points denote
  evaluations and lines connect successive checkpoints.}
  \label{fig:qwen3-training-dynamics}
\end{figure}

\subsection{Trained-checkpoint credit redistribution}
\label{app:credit-examples}

Figure~\ref{fig:credit-examples} makes the redistribution in
Eq.~\eqref{eq:final} concrete using two responses from a replay of the
trained R1-Distill-Qwen-1.5B \method{} checkpoint.  The replay uses the
reported $G=8$, $M=8$ generation, selector, budget, shrinkage, and
cutoff-rescue settings on a fixed cohort of 64 held-out MATH validation
prompts.  The R1 training grader assigns $1$ to a correct boxed answer,
$0.1$ to a boxed but incorrect answer, and $0$ when no valid boxed answer can
be extracted; each interior value therefore estimates expected
verifier reward rather than literal success probability.
This replay visualizes the resulting advantage tensor on concrete responses; Table~\ref{tab:main} provides the efficacy evidence and Figure~\ref{fig:oracle} evaluates selector placement. We select two complementary responses from the fixed cohort to illustrate distinct redistribution patterns. Each has two semantic pivots and three segments of at least 64 tokens. The plotted values use the deployed $M=8$ estimator, and Table~\ref{tab:credit-sensitivity} reports the pooled-$M=32$ sensitivity check.

\begin{table}[H]
\centering
\caption{\textbf{Sensitivity of the displayed token-advantage profiles.}
The pooled-$M=32$ calculation preserves the $(+,-,+)$ pattern in column~(a) and
the terminal-most-negative ordering in column~(b).}
\label{tab:credit-sensitivity}
\small
\setlength{\tabcolsep}{6pt}
\begin{tabular}{lccc}
\toprule
Profile & Segment 1 & Segment 2 & Segment 3 \\
\midrule
Panel (a), deployed $M=8$ & $+0.34$ & $-0.41$ & $+0.72$ \\
Panel (a), pooled $M=32$ & $+0.41$ & $-0.10$ & $+0.53$ \\
Panel (b), deployed $M=8$ & $-1.08$ & $-1.08$ & $-2.89$ \\
Panel (b), pooled $M=32$ & $-1.37$ & $-0.80$ & $-3.24$ \\
\bottomrule
\end{tabular}
\end{table}

The $M=32$ calculation is a descriptive top-up sensitivity check, not an
independent replication or an alternative training configuration.

\paragraph{Pivot-aligned question and response context.}
The questions below omit only their raw diagram code; the cube diagram is verbalized explicitly.
The sampled-response excerpts are verbatim apart from whitespace and mathematical typesetting, ellipses mark omitted spans, and [P1]/[P2] are editorial markers placed at the selector's boundaries.
The selector produced these pivots before the examples were chosen for display; the unabridged responses and probe records are retained in the companion artifact.

\paragraph{Example A: cube paths (correct response).}
\textbf{Question (diagram verbalized).}
How many different three-step paths along the edges of a cube are there that take you from vertex $A$ to vertex $B$?
(A step is from a vertex to an adjacent vertex sharing an edge.)
The diagram labels $A=(0,0,1)$ and $B=(1,1,0)$, opposite cube vertices.
\textbf{Ground truth:} $6$.

\textbf{Sampled response excerpts.}
\begin{quote}
\footnotesize
\setlength{\emergencystretch}{1em}
\textbf{Segment 1 ($+0.34$).}
``\ldots from $A$ $(0,0,1)$ to $B$ $(1,1,0)$, the coordinates need to change from $(0,0,1)$ to $(1,1,0)$.
So, the changes needed are $+1$ in $x$, $+1$ in $y$, and $-1$ in $z$.
But each step can only change one coordinate at a time.
So, we need to find the number of sequences of three moves that result in changing $x$ from $0$ to $1$, $y$ from $0$ to $1$, and $z$ from $1$ to $0$.''

\noindent\textcolor{blue!70!black}{\textbf{[P1: selected paragraph boundary, token 393]}}

\noindent\textbf{Segment 2 ($-0.41$).}
``This is similar to a path counting problem in a grid, but in three dimensions.
Hmm, maybe using recursion or combinatorics.
Alternatively, we can model this as a graph problem where each vertex is a cube vertex, and edges connect adjacent vertices. \ldots
From $A$ $(0,0,1)$, the possible first moves are to $(1,0,1)$, $(0,1,1)$, or $(0,0,0)$.
But $(0,0,0)$ is not $B$, so we need to see which moves lead closer to $B$.
Wait, but actually, we need to reach $B$ in exactly three steps. So''

\noindent\textcolor{blue!70!black}{\textbf{[P2: selected ``So'' boundary, token 545]}}

\noindent\textbf{Segment 3 ($+0.72$).}
``, maybe we can use multinomial coefficients or consider the number of ways to arrange the moves. \ldots
We need one move in $x$, one move in $y$, and one move in $z$, but in some order. \ldots
The number of ways to arrange the moves in $x$, $y$, and $z$ \ldots is $3! = 6$. \ldots
The number of different three-step paths from vertex $A$ to vertex $B$ on a cube is $\boxed{6}$.''
\end{quote}

The first pivot follows the correct reduction to three coordinate changes; the intervening segment detours into a less directed graph search, and the second pivot precedes the return to the permutation argument.
Thus the selected boundaries make the temporary estimated-value dip and recovery in Figure~\ref{fig:credit-examples}, column~(a), directly interpretable.

\paragraph{Example B: parallelogram ratio reversal (incorrect response).}
\textbf{Question (diagram code omitted).}
Let $ABCD$ be a parallelogram.
We have that $M$ is the midpoint of $AB$ and $N$ is the midpoint of $BC$.
The segments $DM$ and $DN$ intersect $AC$ at $P$ and $Q$, respectively.
If $AC=15$, what is $QA$?
\textbf{Ground truth:} $10$.

\textbf{Sampled response excerpts.}
\begin{quote}
\footnotesize
\setlength{\emergencystretch}{1em}
\textbf{Segment 1 ($-1.08$).}
``I think using coordinate geometry might help here. \ldots
Let me place point $A$ at $(0,0)$.
Since it's a parallelogram, let me assume point $B$ is at $(2b,0)$, point $D$ at $(0,2d)$, and then point $C$ would be at $(2b,2d)$. \ldots
$M=(b,0)$. \ldots $N=(2b,d)$.''

\noindent\textcolor{blue!70!black}{\textbf{[P1: selected paragraph boundary, token 383]}}

\noindent\textbf{Segment 2 ($-1.08$).}
``To find $P$, the intersection of $DM$ and $AC$, we set the equations equal \ldots $s=1/3$. \ldots
Similarly, now find $Q$, which is the intersection of $DN$ and $AC$. \ldots
$2d-ds=2ds \Rightarrow 2d=3ds \Rightarrow s=2/3$.
So, $s=2/3$, so $Q$ is at $(4b/3,4d/3)$. \ldots
But the answer seems to be $QA=10$, but I need to confirm.''

\noindent\textcolor{blue!70!black}{\textbf{[P2: selected paragraph boundary, token 919]}}

\noindent\textbf{Segment 3 ($-2.89$).}
``Alternatively, using coordinates, the length of $AC$ is $15$. \ldots
From earlier, $s=1/3$ for $Q$, so $AQ=1/3$ of $AC=5$, and $QC=10$.
Wait, but the question is about $QA$, which is $5$. \ldots
The length of $QA$ is $\boxed{5}$.''
\end{quote}

The first pivot separates coordinate setup from the intersection calculations.
The middle segment reaches the correct $Q$ ratio $s=2/3$ and states $QA=10$; after the second pivot, the response reuses the earlier $s=1/3$ ratio associated with the other intersection and changes its answer to $5$.
This makes the terminal-most-negative allocation in Figure~\ref{fig:credit-examples}, column~(b), consistent with the displayed reasoning transition.
\FloatBarrier
\section{Comparison with related credit-assignment methods}
\label{app:trees}

\begin{table}[H]
\centering
\caption{\textbf{Mechanism comparison.} Cells summarize the cited configurations.}
\label{tab:matrix}
\label{tab:capability}
\footnotesize
\definecolor{tableours}{RGB}{238,246,241}

\begin{minipage}{\linewidth}
\raggedright
\textit{Key.}
$\checkmark$: support;
$\times$: absence;
$\pm$: partial or implementation-qualified support;
\textemdash: not applicable;
$\S$: surrogate-probe family \citep{li2026ssvpo,wang2026grpo}.
``Exact GRPO'' denotes a $K=0$ or $\alpha=0$ tensor setting with fixed accepted
trajectories and matched preprocessing.
\end{minipage}

\vspace{4pt}
\renewcommand{\arraystretch}{1.10}
\setlength{\tabcolsep}{4pt}
\begin{tabularx}{\linewidth}{@{}>{\raggedright\arraybackslash}p{0.34\linewidth}*{3}{>{\centering\arraybackslash}X}@{}}
\toprule
\multicolumn{4}{@{}l@{}}{\textit{(a) Measurement and value estimation}} \\
\midrule
\textbf{Method} & \makecell[c]{\textbf{Budgeted}\\\textbf{spend}} & \makecell[c]{\textbf{Reward-free}\\\textbf{interior pilot}} & \makecell[c]{\textbf{Estimates}\\$V^{\pi_\theta}(s)$} \\
\midrule
GRPO \citep{shao2024deepseekmath} & (no spend) & \textemdash & \textemdash \\
VinePPO \citep{kazemnejad2025vineppo} & $\times$ & \textemdash\ (uniform) & \checkmark \\
SPO-tree \citep{guo2025spo} & $\times$ & $\times$ & $\pm$ \\
TreeRPO \citep{yang2025treerpo} & $\times$ & $\times$ & $\pm$ \\
TreePO \citep{li2025treepo} & $\pm$ & $\times$ & $\pm$ \\
ROSE \citep{zhao2026reinforced} & $\pm$ & \checkmark & $\times$ \\
Surrogate probes$^{\S}$ & $\pm$ & $\pm$ & $\times$ \\
PRM-label MC \citep{luo2024omegaprm} & \checkmark & $\times$ (offline) & \checkmark \\
\midrule
\rowcolor{tableours}
\textbf{\method{}} & \checkmark & \checkmark & \checkmark \\
\bottomrule
\end{tabularx}

\vspace{7pt}

\begin{tabularx}{\linewidth}{@{}>{\raggedright\arraybackslash}p{0.34\linewidth}*{3}{>{\centering\arraybackslash}X}@{}}
\toprule
\multicolumn{4}{@{}l@{}}{\textit{(b) Credit and update behavior}} \\
\midrule
\textbf{Method} & \makecell[c]{\textbf{Token-mean-}\\\textbf{preserving injection}} & \makecell[c]{\textbf{Exact GRPO}\\\textbf{tensor setting}} & \makecell[c]{\textbf{Actor multiset}\\\textbf{unchanged}} \\
\midrule
GRPO \citep{shao2024deepseekmath} & (base) & (base) & \checkmark \\
VinePPO \citep{kazemnejad2025vineppo} & $\times$ & $\times$ & \checkmark \\
SPO-tree \citep{guo2025spo} & $\times$ & $\times$ & $\times$ \\
TreeRPO \citep{yang2025treerpo} & $\times$ & $\times$ & $\times$ \\
TreePO \citep{li2025treepo} & $\times$ & $\times$ & $\times$ \\
ROSE \citep{zhao2026reinforced} & $\times$ & $\times$ & $\times$ \\
Surrogate probes$^{\S}$ & $\times$ & $\times$ & \checkmark \\
PRM-label MC \citep{luo2024omegaprm} & (offline) & (offline) & (offline) \\
\midrule
\rowcolor{tableours}
\textbf{\method{}} & \checkmark & \checkmark & \checkmark \\
\bottomrule
\end{tabularx}
\end{table}

Table~\ref{tab:matrix} highlights \method{}'s combination of budgeted policy-matched measurement, token-mean-preserving credit, exact GRPO recovery, and an unchanged actor batch.

\section{Oracle study details}
\label{app:oracle}

The oracle uses up to five candidate boundaries per trajectory, six head actions with 32 continuations per action, and a 16-continuation rejection-sampled tail stratum when head-action coverage is below $0.9$.
The MATH-precalc and GSM8K artifacts contain 286 and 212 candidates, of which 239 and 104 have finite action-fork variance estimates.
The offline oracle tests whether forward entropy tracks action-fork variance on archived trajectories. It computes entropy from the top 10 returned token log-probabilities and uses an adaptive forward window; the deployed selector uses full-vocabulary entropy with fixed $L=8$. The controlled random-placement ablation in Section~\ref{sec:component-ablations} evaluates the complete deployed placement module.
The archived backbone trajectories use temperature $0.35$ and top-$p$ $0.9$, first-action weights use temperature $1$ and top-$p$ $1$, and forced suffixes use temperature $0.35$ and top-$p$ $0.95$.
For Figure~\ref{fig:oracle}A, we resample prompts with replacement, retain every finite boundary belonging to each sampled prompt (including multiplicities), and recompute the pooled tied-rank Spearman correlation.
There is one oracle trajectory per prompt, so prompt and trajectory clusters coincide; all 100{,}000 replicates are finite.
For panel B, we retain trajectories with at least two finite candidates and positive total oracle variance, average uniformly over ties in the selector maximum, and compare captured variance with the exact conditional random-selection expectation $1/m_i$ for a trajectory with $m_i$ candidates.
The interval resamples the paired selector-minus-random differences by trajectory.
The reported intervals quantify prompt- or trajectory-level variation conditional on the oracle estimates.

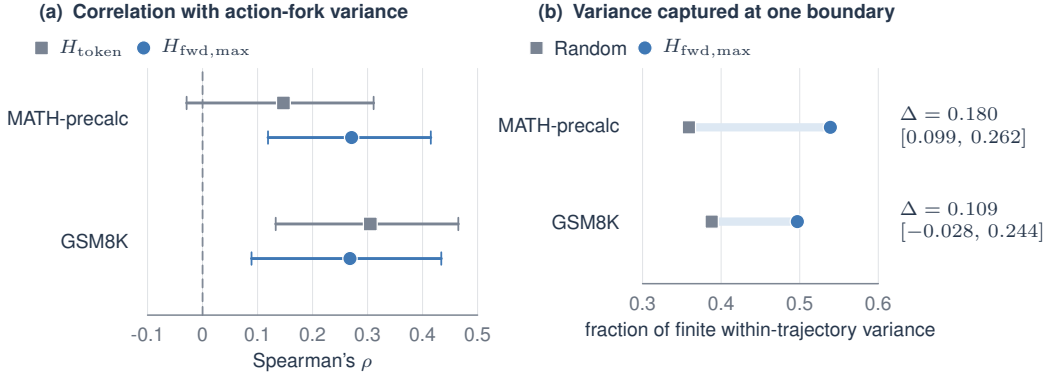
\begin{figure}[H]
  \centering
  \definecolor{OracleInk}{HTML}{26364B}
\definecolor{OracleGray}{HTML}{7A8493}
\definecolor{OracleBlue}{HTML}{3F78B5}
\definecolor{OraclePale}{HTML}{DDE8F3}
\resizebox{\linewidth}{!}{%
\begin{tikzpicture}[
  font=\sffamily\scriptsize,
  line cap=round,
  ci/.style={OracleBlue, line width=1.05pt},
  hmax/.style={circle, fill=OracleBlue, draw=white, line width=0.45pt,
               inner sep=0pt, minimum size=5.2pt},
  htoken/.style={rectangle, fill=OracleGray, draw=white, line width=0.45pt,
                 inner sep=0pt, minimum size=5.0pt},
  random/.style={rectangle, fill=OracleGray, draw=white, line width=0.45pt,
                 inner sep=0pt, minimum size=4.8pt}
]
\begin{scope}
  \node[anchor=west, font=\sffamily\bfseries\scriptsize, text=OracleInk]
    at (-0.10,0.72) {\textbf{(a)}\enspace Correlation with action-fork variance};
  \node[htoken] at (0.05,0.28) {};
  \node[anchor=west, text=OracleInk] at (0.15,0.28) {$H_{\mathrm{token}}$};
  \node[hmax] at (1.35,0.28) {};
  \node[anchor=west, text=OracleInk] at (1.45,0.28) {$H_{\mathrm{fwd,max}}$};

  \foreach \v in {-0.1,0,0.1,0.2,0.3,0.4,0.5} {
    \draw[OracleGray!23, line width=0.35pt]
      ({2.1+7*\v},-3.12) -- ({2.1+7*\v},0.02);
    \node[anchor=north, text=OracleGray!88!black, font=\sffamily\scriptsize]
      at ({2.1+7*\v},-3.15) {\v};
  }
  \draw[OracleInk!58, densely dashed, line width=0.65pt]
    ({2.1+7*0},-3.12) -- ({2.1+7*0},0.02);
  \draw[OracleInk!60, line width=0.55pt] (1.40,-3.12) -- (5.60,-3.12);
  \node[anchor=north, text=OracleInk, font=\sffamily\scriptsize]
    at (3.50,-3.48) {Spearman's $\rho$};

  \node[anchor=east, text=OracleInk] at (1.28,-0.64) {MATH-precalc};
  \node[anchor=east, text=OracleInk] at (1.28,-2.18) {GSM8K};

  \draw[OracleGray, line width=1.05pt] ({2.1+7*(-0.029)},-0.42) -- ({2.1+7*0.311},-0.42);
  \draw[OracleGray, line width=0.65pt] ({2.1+7*(-0.029)},-0.50) -- ({2.1+7*(-0.029)},-0.34);
  \draw[OracleGray, line width=0.65pt] ({2.1+7*0.311},-0.50) -- ({2.1+7*0.311},-0.34);
  \node[htoken] at ({2.1+7*0.147},-0.42) {};
  \draw[ci] ({2.1+7*0.119},-0.86) -- ({2.1+7*0.415},-0.86);
  \draw[ci, line width=0.65pt] ({2.1+7*0.119},-0.94) -- ({2.1+7*0.119},-0.78);
  \draw[ci, line width=0.65pt] ({2.1+7*0.415},-0.94) -- ({2.1+7*0.415},-0.78);
  \node[hmax] at ({2.1+7*0.271},-0.86) {};

  \draw[OracleGray, line width=1.05pt] ({2.1+7*0.133},-1.96) -- ({2.1+7*0.465},-1.96);
  \draw[OracleGray, line width=0.65pt] ({2.1+7*0.133},-2.04) -- ({2.1+7*0.133},-1.88);
  \draw[OracleGray, line width=0.65pt] ({2.1+7*0.465},-2.04) -- ({2.1+7*0.465},-1.88);
  \node[htoken] at ({2.1+7*0.305},-1.96) {};
  \draw[ci] ({2.1+7*0.089},-2.40) -- ({2.1+7*0.434},-2.40);
  \draw[ci, line width=0.65pt] ({2.1+7*0.089},-2.48) -- ({2.1+7*0.089},-2.32);
  \draw[ci, line width=0.65pt] ({2.1+7*0.434},-2.48) -- ({2.1+7*0.434},-2.32);
  \node[hmax] at ({2.1+7*0.268},-2.40) {};
\end{scope}

\begin{scope}[shift={(6.25,0)}]
  \node[anchor=west, font=\sffamily\bfseries\scriptsize, text=OracleInk]
    at (0,0.72) {\textbf{(b)}\enspace Variance captured at one boundary};
  \node[random] at (0.10,0.28) {};
  \node[anchor=west, text=OracleInk] at (0.20,0.28) {Random};
  \node[hmax] at (1.48,0.28) {};
  \node[anchor=west, text=OracleInk] at (1.58,0.28) {$H_{\mathrm{fwd,max}}$};

  \foreach \v in {0.3,0.4,0.5,0.6} {
    \draw[OracleGray!23, line width=0.35pt]
      ({1.45+10*(\v-0.3)},-2.72) -- ({1.45+10*(\v-0.3)},0.02);
    \node[anchor=north, text=OracleGray!88!black, font=\sffamily\scriptsize]
      at ({1.45+10*(\v-0.3)},-2.75) {\v};
  }
  \draw[OracleInk!60, line width=0.55pt] (1.45,-2.72) -- (4.45,-2.72);
  \node[anchor=north, text=OracleInk, font=\sffamily\scriptsize]
    at (2.95,-3.08) {fraction of finite within-trajectory variance};

  \node[anchor=east, text=OracleInk] at (1.28,-0.73) {MATH-precalc};
  \node[anchor=east, text=OracleInk] at (1.28,-1.93) {GSM8K};

  \draw[OraclePale, line width=3.0pt]
    ({1.45+10*(0.359-0.3)},-0.73) -- ({1.45+10*(0.539-0.3)},-0.73);
  \node[random] at ({1.45+10*(0.359-0.3)},-0.73) {};
  \node[hmax] at ({1.45+10*(0.539-0.3)},-0.73) {};
  \node[anchor=west, align=left, text=OracleInk, font=\sffamily\scriptsize]
    at (4.60,-0.73) {$\Delta=0.180$\\$[0.099,\,0.262]$};

  \draw[OraclePale, line width=3.0pt]
    ({1.45+10*(0.388-0.3)},-1.93) -- ({1.45+10*(0.497-0.3)},-1.93);
  \node[random] at ({1.45+10*(0.388-0.3)},-1.93) {};
  \node[hmax] at ({1.45+10*(0.497-0.3)},-1.93) {};
  \node[anchor=west, align=left, text=OracleInk, font=\sffamily\scriptsize]
    at (4.60,-1.93) {$\Delta=0.109$\\$[-0.028,\,0.244]$};
\end{scope}
\end{tikzpicture}%
}
  \caption{\textbf{Offline selector diagnostics.}
  \emph{(a)} Rank correlation between reward-free signals and action-fork
  variance.  \emph{(b)} Oracle variance captured by the top-ranked versus a
  uniformly selected boundary.  Error bars and brackets show the bootstrap
  intervals described in the text.}
  \label{fig:oracle}
\end{figure}

\FloatBarrier
\section{Recipes, hyperparameters, and reproducibility}
\label{app:repro}

\paragraph{Training data.}
The original MATH benchmark contains 12{,}500 competition problems, split into
7{,}500 training and 5{,}000 test problems \citep{hendrycks2021measuring}.
The DeepSeekMath-7B-SFT and R1-Distill-Qwen-1.5B regimes follow VinePPO's
OpenAI/PRM800K repartition \citep{lightman2024lets,kazemnejad2025vineppo}:
the original 7{,}500 training problems and 4{,}500 test problems form a
12{,}000-problem development pool, while the remaining 500 uniformly sampled
test problems form MATH-500.  VinePPO reserves another 500 development
problems for validation, leaving 11{,}500 source training problems.  DeepSeekMath
uses VinePPO's prompt wrapper; R1 renders the same identities with the SPO/R1
template.  Qwen3-4B-Base follows ROSE \citep{zhao2026reinforced}, using the
original 7{,}500-problem MATH training split from
\texttt{DigitalLearningGmbH/MATH-lighteval}.  ROSE's 2{,}048-token prompt
filter removes none under our non-thinking Qwen3 rendering.

\paragraph{Prompt and answer templates.}
We retain the prompt convention paired with each model family rather than imposing one cross-family wrapper.
DeepSeekMath receives VinePPO's baked user string, \texttt{[MATH\_TASK] Problem:} followed by the problem and \texttt{Solution:}; a pass-through chat template adds no system instruction, and correctness is read from the generated final answer by the verl/Hendrycks grader.
R1 receives the literal SPO/R1 wrapper: it asks the model to solve efficiently and clearly, think step by step, and end with \texttt{Therefore, the final answer is: \textbackslash boxed\{ANSWER\}. I hope it is correct}; the rendered assistant prefix opens \texttt{<think>}, and the GRESO/math-verify path grades the final boxed answer.
Qwen3 retains ROSE's system message requesting step-by-step reasoning and a final answer in \texttt{\textbackslash boxed\{\}}, followed by the raw problem as the user message.
Following ROSE's \emph{no-CoT} (non-thinking) configuration \citep{zhao2026reinforced}, we render this conversation with Qwen3's native chat template and \texttt{enable\_thinking=false}; this suppresses the model-specific \texttt{<think>... </think>} block even though the retained ROSE system message still requests a worked solution.
The frozen ROSE evaluator extracts the last boxed expression and applies normalized exact match.

\FloatBarrier
\begin{table}[H]
\centering
\caption{\textbf{Shared \method{} acquisition and estimator settings.}
These values apply to the three headline \method{} arms in Table~\ref{tab:main}.}
\label{tab:probepo-shared-recipe}
\small
\renewcommand{\arraystretch}{1.08}
\setlength{\tabcolsep}{8pt}
\begin{tabular}{lc}
\toprule
Parameter & Value \\
\midrule
Backbone group size $G$ & 8 responses per prompt \\
Group-wide probe budget $B$ & 40 continuation rollouts \\
Continuations per selected pivot $M$ & 8 \\
Shrinkage pseudo-count $n_0$ & 4 \\
Residual weight $\alpha$ & 1 \\
Stage-1 score & $H_{\max}$ with forward window $L=8$ \\
Per-trajectory shortlist & $\max\{1,\operatorname{round}(0.30J)\}$ of $J$ candidates \\
Positional weight $w_{\mathrm{pos}}$ & uniform ($1$) \\
\bottomrule
\end{tabular}
\end{table}

\begin{table}[H]
\centering
\caption{\textbf{Optimization and batching for the headline \method{} runs.}
Batch and mini-batch are prompt counts.}
\label{tab:probepo-training-recipe}
\footnotesize
\renewcommand{\arraystretch}{1.08}
\setlength{\tabcolsep}{4.2pt}
\begin{tabular}{@{}lccc@{}}
\toprule
Parameter & \makecell[c]{DeepSeekMath\\7B} & \makecell[c]{R1-Distill-Qwen\\1.5B (2K context)} & \makecell[c]{Qwen3\\4B} \\
\midrule
Learning rate & $5{\times}10^{-7}$ & $1{\times}10^{-6}$ & $1{\times}10^{-6}$ \\
Schedule
& \makecell[c]{constant;\\no warmup}
& \makecell[c]{3\% linear warmup,\\then constant}
& \makecell[c]{constant;\\no warmup} \\
Batch / mini & $64/8$ & $64/16$ & $512/8$ \\
PPO epochs & 1 & 1 & 1 \\
PPO clip ratio & 0.2 & 0.2 & 0.2 \\
KL regularization
& \makecell[c]{loss-side $10^{-3}$\\\texttt{low\_var\_kl}}
& \makecell[c]{loss-side $10^{-3}$\\\texttt{low\_var\_kl}}
& \makecell[c]{loss-side $10^{-3}$\\\texttt{low\_var\_kl}} \\
Reward-side KL & off & off & off \\
Loss aggregation & \texttt{token-mean} & \texttt{token-mean} & \makecell[c]{\texttt{seq-mean-}\\\texttt{token-sum-norm}} \\
\bottomrule
\end{tabular}
\end{table}

\begin{table}[H]
\centering
\caption{\textbf{Decoding, context, and training horizons for the headline \method{} runs.}
Probe continuations inherit the backbone decoding parameters.}
\label{tab:probepo-decoding-recipe}
\footnotesize
\renewcommand{\arraystretch}{1.08}
\setlength{\tabcolsep}{5.5pt}
\begin{tabular}{@{}lccc@{}}
\toprule
Parameter & \makecell[c]{DeepSeekMath\\7B} & \makecell[c]{R1-Distill-Qwen\\1.5B (2K context)} & \makecell[c]{Qwen3\\4B} \\
\midrule
Train/probe temperature, top-$p$ & $0.8,\ 0.95$ & $1.0,\ 1.0$ & $1.0,\ 1.0$ \\
Maximum prompt tokens & 512 & 512 & 2{,}048 \\
Configured response-token limit & 1{,}024 & 4{,}096 & 4{,}096 \\
Rollout context limit & 4{,}096 & \makecell[c]{2{,}048\\(context binding)} & 6{,}144 \\
Training horizon (steps) & 1{,}000 & 1{,}000 & 27 \\
\bottomrule
\end{tabular}
\end{table}

\FloatBarrier

The launcher field \texttt{budget\_per\_trajectory=5} is denominated in continuation rollouts rather than selected pivots; with $G=8$, it gives the group-wide $B=40$ budget in Table~\ref{tab:probepo-shared-recipe}, sufficient for at most five $M=8$ pivots per prompt group.

\paragraph{Compute resources.}
Each headline experiment used one node with eight NVIDIA H200 GPUs; every controlled Qwen3 arm used one node with eight NVIDIA B200 GPUs.

\paragraph{Evaluation protocol.}
DeepSeekMath uses the verl grader with a 4{,}096-token evaluation context, R1-Distill-Qwen-1.5B uses math-verify with a 2{,}048-token context, and Qwen3 uses the frozen ROSE last-boxed exact-match evaluator with a 6{,}144-token total context and a 4{,}096-token generation cap.
The GRPO reference is matched to each headline \method{} arm on context, optimizer microbatch, rollout temperature, KL coefficient, and learning rate.
The controlled Qwen3 arms use the same frozen six-dataset evaluation protocol.
The retained manifest maps every table row to its run and evaluator aggregate.

\paragraph{Collection.}
Across the controlled Qwen3 arms, we hold prompt batches, advantage normalization, coefficient clipping, and group handling fixed.

\paragraph{Checkpoint protocol.}
All six scores in a row of Table~\ref{tab:main} come from the same checkpoint; AMC or MinervaMath add-on evaluations are joined only at that same checkpoint.
This removes benchmark-specific test-set checkpoint selection, and all rows within a setting share the same training horizon: step 1{,}000 for DeepSeekMath and R1, and step 27 for Qwen3.
The component arms are compared at their first common checkpoint.
The equal-time comparison uses checkpoints with nearly identical cumulative logged training time of approximately 11.14 hours.

\end{document}